\documentclass[twoside]{article}

\usepackage{subfiles} 
\usepackage{xr}
\usepackage{hyperref}       
\usepackage{url}            
\usepackage{booktabs}       
\usepackage{amsfonts}       
\usepackage{nicefrac}       
\usepackage{microtype}      

\usepackage{amsmath}
\usepackage{subcaption, graphicx}
\usepackage{appendix}
\usepackage[ruled,vlined]{algorithm2e}
\usepackage{amsthm}

\usepackage{floatrow}
\newfloatcommand{capbtabbox}{table}[][\FBwidth]
\usepackage{blindtext}
\usepackage{enumitem}
\usepackage{bbold}

\newtheorem{theorem}{Theorem}
\newtheorem{lemma}{Lemma}
\newtheorem{definition}{Definition}
\newtheorem{proposition}{Proposition}
\newtheorem{corollary}{Corollary}

\usepackage[accepted]{aistats2021}
\usepackage[round]{natbib}

\begin{document}

%
\runningtitle{Deep Probabilistic Accelerated Evaluation (Deep-PrAE)}

%

\twocolumn[

\aistatstitle{Deep Probabilistic Accelerated Evaluation: A Robust Certifiable Rare-Event Simulation Methodology for Black-Box \\Safety-Critical Systems}

\aistatsauthor{ Mansur Arief \And Zhiyuan Huang \And  Guru K.S. Kumar \And Yuanlu Bai }

\aistatsaddress{ Carnegie Mellon University \And  Carnegie Mellon University \And Carnegie Mellon University \And Columbia University} 

\aistatsauthor{Shengyi He \And Wenhao Ding \And Henry Lam \And Ding Zhao }

\aistatsaddress{Columbia University \And Carnegie Mellon  University \And Columbia University \And Carnegie Mellon University} ]

\runningauthor{Arief, Huang, Kumar, Bai, He, Ding, Lam, Zhao}

\begin{abstract}
  Evaluating the reliability of intelligent physical systems against rare safety-critical events poses a huge testing burden for real-world applications. Simulation provides a useful platform to evaluate the extremal risks of these systems before their deployments. Importance Sampling (IS), while proven to be powerful for rare-event simulation, faces challenges in handling these learning-based systems due to their black-box nature that fundamentally undermines its efficiency guarantee, which can lead to under-estimation without diagnostically detected. We propose a framework called Deep Probabilistic Accelerated Evaluation (Deep-PrAE) to design statistically guaranteed IS, by converting black-box samplers that are versatile but could lack guarantees, into one with what we call a relaxed efficiency certificate that allows accurate estimation of bounds on the safety-critical event probability. We present the theory of Deep-PrAE that combines the dominating point concept with rare-event set learning via deep neural network classifiers, and demonstrate its effectiveness in numerical examples including the safety-testing of an intelligent driving algorithm.
\end{abstract}

\section{Introduction}

The unprecedented deployment of intelligent physical systems on many real-world applications comes with the need for safety validation and certification \citep{kalra2016driving,koopman2017autonomous, uesato2018rigorous}. For systems that interact with humans and are potentially safety-critical - which can range from medical systems to self-driving cars and personal assistive robots - it is imperative to rigorously assess their risks before their full-scale deployments. The challenge, however, is that these risks are often associated precisely to how AI reacts in rare and catastrophic scenarios which, by their own nature, are not sufficiently observed. 

The challenge of validating the safety of intelligent systems described above is, unfortunately, insusceptible to traditional test methods. In the self-driving context, for instance, the goal of validation is to ensure the AI-enabled system reduces human-level accident rate (in the order of 1.5 per $10^{8}$ miles of driving), thus delivering enhanced safety promise to the public \citep{Evan2016FatalPerfect,kalra2016driving, PreliminaryHWY16FH018}. Formal verification, which mathematically analyzes and verifies autonomous design, faces challenges when applied to black-box or complex models due to the lack of analytic tractability to formulate failure cases or consider all execution trajectories \citep{clarke2018handbook}. Automated scenario selection approaches generate test cases based on domain knowledge \citep{wegener2004evaluation} or adaptive searching algorithms (such as adaptive stress testing; \citealt{koren2018adaptive}), which is more implementable but falls short of rigor. Test matrix approaches, such as Euro NCAP \citep{national2007new}, use prepopulated test cases extracted from crash databases, but they only contain historical human-driver information. The closest analog to the latter for self-driving vehicles is ``naturalistic tests'', which means placing them in real-world environments and gathering observations. This method, however, is economically prohibitive because of the rarity of the target conflict events \citep{zhao2017accelerated, arief2018accelerated, claybrook2018autonomous,o2018scalable}. 

Because of all these limitations, simulation-based tests surface as a powerful approach to validate complex black-box designs \citep{corso2020survey}. This approach operates by integrating the target intelligent algorithm into an interacting virtual simulation platform that models the surrounding environment. By running enough Monte Carlo sampling of this (stochastic) environment, one hopes to observe catastrophic conflict events and subsequently conduct statistical analyses. This approach is flexible and scalable, as it hinges on building a virtual environment instead of physical systems, and provides a probabilistic assessment on the occurrences and behaviors of safety-critical events  \citep{ koopman2018toward}.  
 
 Nonetheless, similar to the challenge encountered by naturalistic tests, because of their rarity, safety-critical events are seldom observed in the simulation experiments. In other words, it could take an enormous amount of Monte Carlo simulation runs to observe one ``hit'', and this in turn manifests statistically as a large estimation variance per simulation run relative to the target probability of interest (i.e., the so-called \emph{relative error}; \citealt{l2010asymptotic}). This problem, which is called rare-event simulation \citep{bucklew2013introduction}, is addressed conventionally under the umbrella of variance reduction, which includes a range of techniques from importance sampling (IS) \citep{juneja2006rare,blanchet2012state} to multi-level splitting \citep{glasserman1999multilevel,villen1994restart}. Typically, to ensure the relative error is dramatically reduced, one has to analyze the underlying model structures to gain understanding of the rare-event behaviors, and leverage this knowledge to design good Monte Carlo schemes \citep{juneja2006rare,dean2009splitting}. For convenience, we call such relative error reduction guarantee an \emph{efficiency certificate}.
 
Our main focus of this paper is on rare-event problems with the underlying model unknown or too complicated to support analytical tractability. In this case, traditional variance reduction approaches may fail to provide an efficiency certificate. Moreover, we will explain how some existing ``black-box'' variance reduction techniques, while versatile and powerful, could lead to dangerous \emph{under-estimation} of a rare-event probability without detected diagnostically due to a lack of efficiency certificate. This motivates us to study a framework to convert these black-box methods into one that has rigorous certificate. More precisely, our framework consists of three ingredients:

\textbf{Relaxed efficiency certificate: }We shift the estimation of target rare-event probability to an upper (and lower) bound, in a way that supports the integration of learning errors into variance reduction without giving up estimation correctness.
 
\textbf{Set-learning with one-sided error: }We design learning algorithms based on deep neural network classifier to create outer (or inner) approximations of rare-event sets. This classifier has a special property that, under a geometric property called orthogonal monotonicity, it exhibits zero false negative rates. 

\textbf{Deep-learning-based IS: }With the deep-learning based rare-event set approximation, we search the so-called \emph{dominating points} in rare-event analysis to create IS that achieves the relaxed efficiency certificate.

We call our framework consisting of the three ingredients above \emph{Deep Probabilistic Accelerated Evaluation (Deep-PrAE)}, where ``Accelerated Evaluation'' follows terminologies in recent approaches for the safety-testing of autonomous vehicles \citep{zhao2016accelerated, 8116682}. In the set-learning step in Deep-PrAE, the samples fed into our deep classifier can be generated by any black-box algorithms including the cross-entropy (CE) method \citep{de2005tutorial,rubinstein2013cross} and particle approaches such as adaptive multi-level splitting (AMS) \citep{au2001estimation,cerou2007adaptive, webb2018statistical}. Deep-PrAE turns these samples into an IS with an efficiency certificate against undetected under-estimation. Our approach is robust in the sense that it provides a tight bound for the target rare-event probability if the underlying classifier is expressive enough, while it still provides a correct, though conservative, bound if the classifier is weak. To our best knowledge, such type of guarantees and robustness features is the first of its kind in the rare-event simulation literature, and we envision our work to lay the foundation for further improvements to design certified methods for evaluating more sophisticated intelligent designs. 

\section{Statistical Challenges in Black-Box Rare-Event Simulation}\label{sec:existing}

Our evaluation goal is the probabilistic assessment of a complex physical system invoking rare but catastrophic events in a stochastic environment. For concreteness, we write this rare-event probability $\mu=P(X\in\mathcal S_\gamma)$. Here $X$ is a random vector in $\mathbb R^d$ that denotes the environment, and is distributed according to $p$. $\mathcal S_\gamma$ denotes a safety-critical set on the interaction between the physical system and the environment. The ``rarity'' parameter $\gamma\in\mathbb R$ is considered a large number, with the property that as $\gamma\to\infty$, $\mu\to0$ (Think of, e.g., $\mathcal S_\gamma=\{x:f(x)\geq\gamma\}$ for some risk function $f$ and exceedance threshold $\gamma$). We will work with Gaussian $p$ for the ease of analysis, but our framework is more general (i.e., applies to Gaussian mixtures and other light-tailed distributions). Here, we explain intuitively the main concepts and challenges in black-box rare-event simulation, leaving the details to Appendix \ref{append:challenge}.

\textbf{Monte Carlo Efficiency. }Suppose we use a Monte Carlo estimator $\hat\mu_n$ to estimate $\mu$, by running $n$ simulation runs in total. Since $\mu$ is tiny, the error of a meaningful estimation must be measured in relative term, i.e., we would like 
\begin{equation}
P(|\hat\mu_n-\mu|>\epsilon \mu)\leq\delta\label{efficiency certificate}
\end{equation}where $\delta$ is some confidence level (e.g., $\delta=5\%$) and $0<\epsilon<1$. 

Suppose that $\hat\mu_n$ is unbiased and is an average of $n$ i.i.d. simulation runs, i.e., $\hat\mu_n=(1/n)\sum_{i=1}^n Z_i$ for some random unbiased output $Z_i$. We define the \emph{relative error} $RE=Var(Z_i)/\mu^2$ as the ratio of variance (per-run) and squared mean. Importantly, to attain \eqref{efficiency certificate}, a sufficient condition is $n\geq RE/(\delta\epsilon^2)$. So, when $RE$ is large, the required Monte Carlo size is also large.

\textbf{Challenges in Naive Monte Carlo. }Let $Z_i=I(X_i\in\mathcal S_\gamma)$ where $I(\cdot)$ denotes the indicator function, and $X_i$ is an i.i.d. copy of $X$. Since $Z_i$ follows a Bernoulli distribution, $RE=(1-\mu)/\mu$. Thus, the required $n$  scales linearly in $1/\mu$ (when $\mu$ is tiny). This demanding condition is a manifestation of the difficulty in hitting $\mathcal S_\gamma$. In the standard large deviations regime \citep{dembo2010large,dupuis2011weak} where $\mu$ is exponentially small in $\gamma$, the required Monte Carlo size $n$ would grow \emph{exponentially} in $\gamma$. 

\textbf{Variance Reduction. }The severe burden when using naive Monte Carlo motivates techniques to drive down $RE$. 
First we introduce the following notion:

\begin{definition}
We say an estimator $\hat\mu_n$ satisfies an efficiency certificate to estimate $\mu$ if it achieves \eqref{efficiency certificate} with $n=\tilde{O}(\log(1/\mu))$, for given $0<\epsilon,\delta<1$.
\end{definition}
In the above, $\tilde O(\cdot)$ denotes a polynomial growth in $\cdot$. If $\hat\mu_n$ is constructed from $n$ i.i.d.~samples, then the efficiency certificate can be attained with $RE=\tilde O(\log(1/\mu))$. Note that in the large deviations regime, the sample size $n$ used in a certifiable estimator is reduced from exponential in $\gamma$ in naive Monte Carlo to \emph{polynomial} in $\gamma$. 

\emph{Importance sampling (IS)} is a prominent technique to achieve efficiency certificate \citep{glynn1989importance}. IS generates $X$ from another distribution $\tilde p$ (called IS distribution), and outputs $\hat\mu_n=(1/n)\sum_{i=1}^n L(X_i)I(X_i\in\mathcal S_\gamma)$ where $L=dp/d\tilde p$ is the likelihood ratio, or the Radon-Nikodym derivative, between $p$ and $\tilde p$. Via a change of measure, it is easy to see that $\hat\mu_n$ is unbiased for $\mu$. The key is to control its $RE$ by selecting a good $\tilde p$. This requires analyzing the behavior of the likelihood ratio $L$ under the rare event, and in turn understanding the rare-event sample path dynamics \citep{juneja2006rare}. 

\textbf{Perils of Black-Box Variance Reduction Algorithms. }
Unfortunately, in black-box settings where complete model knowledge and analytical tractability are unavailable, the classical IS methodology faces severe challenges. To explain this, we first need to understand how efficiency certificate can be obtained based on the concept of \emph{dominating points}. From now on, we consider input $X\in\mathbb{R}^d$ from a Gaussian distribution $N(\lambda,\Sigma)$ where $\Sigma$ is positive definite. 
\begin{definition}\label{dominating point def}
A set $A_\gamma\subset \mathbb{R}^d$ is a \emph{dominating set} for the set $\mathcal S_\gamma\subset \mathbb{R}^d$ associated with the distribution $N(\lambda,\Sigma)$ if for any $x\in\mathcal S_\gamma$, there exists at least one $a\in A_\gamma$ such that $(a-\lambda)^T\Sigma^{-1}(x-a)\geq0$. Moreover, this set is minimal in the sense that if any point in $A_\gamma$ is removed, then the remaining set no longer satisfies the above condition. We call any point in $A_\gamma$ a dominating point.
\end{definition}

The dominating set comprises the ``corner'' cases where the rare event occurs \citep{sadowsky1990large}. In other words, each dominating point $a$ encodes, in a local region, the most likely scenario should the rare event happen, and this typically corresponds to the highest-density point in this region. Locality here refers to the portion of the rare-event set that is on one side of the hyperplane cutting through $a$ (see Figure \ref{fig:illustration}(a)). 

Intuitively, to increase the frequency of hitting the rare-event set (and subsequently to reduce variance), an IS would translate the distributional mean from $\lambda$ to the global highest-density point in the rare-event set. The delicacy, however, is that this is \emph{insufficient} to control the variance, due to the ``overshoots'' arising from sampling randomness. In order to properly control the overall variance, one needs to divide the rare-event set into local regions governed by dominating points, and using a mixture IS distribution that accounts for \emph{all} of them. This approach gives a certifiable IS, described as follows:

\begin{figure}[t]
  \begin{subfigure}{.48\textwidth}
  \centering
    \includegraphics[width=\textwidth]{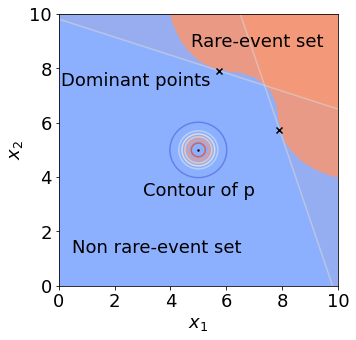}
    \caption{}
  \end{subfigure}%
  \begin{subfigure}{.48\textwidth}
  \centering
    \includegraphics[width=\textwidth]{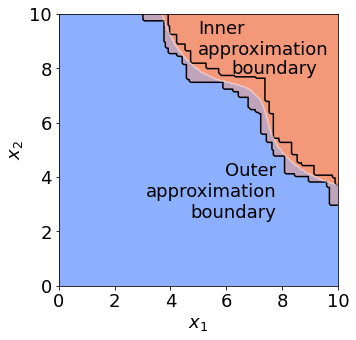}
    \caption{}
  \end{subfigure}
  
    \begin{subfigure}{.48\textwidth}
  \centering
    \includegraphics[width=\textwidth]{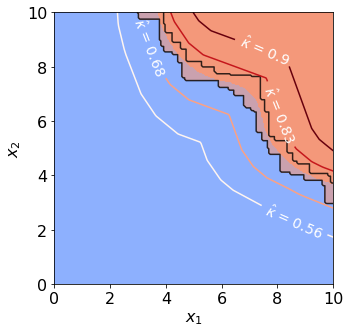}
    \caption{}
  \end{subfigure}
  \begin{subfigure}{.48\textwidth}
  \centering
    \includegraphics[width=\textwidth]{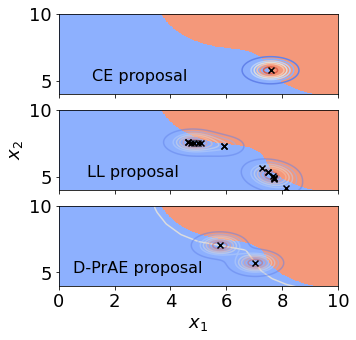}
    \caption{}
  \end{subfigure}
  
\caption{(a) An example of $\mathcal S_\gamma$ with two dominating points (b) Outer- and inner- approximations of $\mathcal S_\gamma$ (c) $\hat \kappa$ tuning for Stage 1 Alg.  \ref{algo:stage1}  (d) IS proposals: too few dominating points for CE with too simple parametric class, too many for LL, and a balance for Deep-PrAE.}
\label{fig:illustration}
\end{figure}

\begin{theorem}[Certifiable IS]
Suppose $\mathcal S_{\gamma}=\bigcup_j \mathcal S_{\gamma}^j$, where each $\mathcal S_\gamma^j$ is a ``local'' region corresponding to a dominating point $a_j\in A_\gamma$ associated with the distribution $N(\lambda,\Sigma)$, with conditions stated precisely in Theorem \ref{general IS} in the Appendix. Then the IS estimator constructed by $n$ i.i.d. outputs drawn from the IS distribution $\sum_j \alpha_j N(a_j,\Sigma)$ achieves an efficiency certificate in estimating $\mu=P(X\in \mathcal S_\gamma)$. \label{general IS simplified}
\end{theorem}

On the contrary, if the Gaussian (mixture) IS distribution misses any of the dominating points, then the resulting estimate may be utterly unreliable for two reasons. First, not only that efficiency certificate may fail to hold, but its RE can be arbitrarily large. Second, even more dangerously, this poor performance can be empirically hidden and leads to a systematic \emph{under-estimation} of the rare-event probability without being detected. In other words, in a given experiment, we may observe a reasonable empirical relative error (i.e., sample variance over squared sample mean), yet the estimate is much lower than the correct value. These are revealed in the following example:

\begin{theorem}[Perils of under-estimation]
Suppose we estimate $\mu=P(X\geq\gamma\text{ or }X\leq -k\gamma)$ where $X\sim p=N(0,1)$ and $0<k<3$. We choose $\tilde p=N(\gamma,1)$ as the IS distribution to obtain $\hat\mu_n$. Then 1) The relative error of $\hat\mu_n$ grows exponentially in $\gamma$. 2) If $n$ is polynomial in $\gamma$, we have $P\left(\left|\hat\mu_n-\bar{\Phi}(\gamma)\right|>\varepsilon\bar{\Phi}(\gamma)\right)=O\left(\frac{\gamma}{n\varepsilon^2}\right)$ for any $\varepsilon>0$ where $\bar\Phi(\gamma)=P(X\geq\gamma)<\mu$, and the empirical relative error $=O(n^2)$ with probability higher than $1-1/2^n$.

\label{counterexample}
\end{theorem}

The second conclusion in Theorem \ref{counterexample} implies that the estimator $\hat\mu_n$, built from an IS with a missed dominating point, systematically under-estimates the target $\mu$, yet with high probability its empirical relative error grows only polynomially in $\gamma$, thus wrongly fooling the user that the estimator is efficient.

With this, we now explain why using black-box variance reduction algorithms can be dangerous - in the sense of not having an efficiency certificate and, more importantly, the risk of an unnoticed systematic under-estimation. In the literature, there are two lines of techniques that apply to black-box problems. The first line is the CE method, which uses optimization to search for a good parametrization over a parametric class of IS. The objective criteria include the cross-entropy (with an oracle-best zero-variance IS distribution; \citealt{rubinstein2013cross,de2005tutorial}) and estimation variance \citep{arouna2004adaptative}. Without closed-form expressions, and also to combat the rare-event issue, one typically solves a sequence of empirical optimization problems, starting from a ``less rare'' problem (i.e., smaller $\gamma$) and gradually increasing the rarity with updated empirical objectives using better IS samples. Achieving efficiency requires both a sufficiently expressive parametric IS class and parameter convergence (so that at the end all the dominating points are accounted for).
The second line of methods is the multi-level splitting or subsimulation \citep{au2001estimation,cerou2007adaptive}, a particle method in lieu of IS, which relies on enough mixing of descendant particles. Full analyses on these methods to reach efficiency certificate appear challenging, and without one the estimators could be under-estimated, and without detected, as illustrated in Theorem \ref{counterexample}. We discuss more details of CE and AMS in  Appendix \ref{app:ce_ams}. 

Note that there are other variants of CE and AMS. The former include enhanced CE such as Markov chain IS \citep{botev2013markov,botev2016semiparametric,grace_kroese_sandmann_2014}, neural network IS \citep{muller2019neural} and nonparametric CE \citep{rubinstein2005stochastic}.

The latter include RESTART which works similarly as subset simulation and splitting but performs a number of simulation retrials after entering regions with a higher importance function value \citep{VILLENALTAMIRANO2010156}. Similar to standard CE and AMS, these methods also face challenges in satisfying an efficiency certificate. 

Lastly, we briefly review several other methods with guarantees similar to our efficiency certificate, but relies heavily on structual knowledge. The first one is large-deviations-based IS including sequential exponential tilting \citep{Bucklew2004,Asmussen2007,Siegmund1976importance} and mixture-based proposals \citep{chen2019efficient}. Another method, which is especially powerful for heavy tailed problems, is conditional Monte Carlo which reduces the variance by sampling conditional on some auxiliary random variables \citep{asmussen2006improved}.

Compared to existing methods as reviewed above, our novelty is to tackle black-box problems while sustaining a correctness guarantee, via a new certificate and a careful integration of set-learning with the dominating point machinery.

\section{The Deep Probabilistic Accelerated Evaluation Framework}\label{sec:Deep-PrAE}

 We propose the Deep-PrAE framework to overcome the challenges faced by black-box variance reduction algorithms. This framework comprises two stages: First is to learn the rare-event set from a first-stage sample batch, by viewing set learning as a classification task. These first-stage samples can be drawn from any rare-event sampling methods including CE and AMS. Second is to apply an efficiency-certified IS on the rare-event probability over the learned set. Algorithm \ref{algo:stage1} shows our main procedure. The key to achieving an ultimate efficiency certificate lies in how we learn the rare-event set in Stage 1, which requires two properties: 

 \paragraph{Small one-sided generalization error: }``One-sided'' generalization error here means the learned set is either an outer or an inner approximation of the unknown true rare-event set, with probability 1. Converting this into a classification, this means the false negative (or positive) rate is exactly 0. ``Small'' here then refers to the other type of error being controlled.

\textbf{Decomposability: }The learned set is decomposable according to dominating points in the form of Theorem \ref{general IS simplified}, so that an efficient mixture IS can apply.

\begin{algorithm}[h]
\KwIn{Black-box evaluator $I(\cdot\in\mathcal S_\gamma)$, initial Stage 1 samples $\{(\tilde X_i, Y_i) \}_{i=1,\ldots,n_1}$, Stage 2 sampling budget $n_2$, input distribution $N(\lambda,\Sigma)$.}
\KwOut{IS estimate $\hat\mu_n$.}

\nl \textbf{Stage 1 (Set Learning):}\\
\nl Train classifier with positive decision region $\overline{\mathcal  S}_\gamma^{\kappa}=\{x:\hat{g}(x) \geq\kappa\}$ using $\{(\tilde X_i, Y_i) \}_{i=1,\ldots,n_1}$;\\
\nl Replace $\kappa$ by $\hat\kappa=\max\{\kappa\in\mathbb R: (\overline{\mathcal  S}_\gamma^{\kappa})^c\subset\mathcal H(T_0)\}$;\\

\nl \textbf{Stage 2 (Mixture IS based on Searched dominating points):}\\
\nl Start with $\hat A_\gamma = \emptyset$;\\

\nl {\bf While } $\{x: \hat g(x) \geq \hat \kappa , (x^*_j-\lambda)^T\Sigma^{-1}(x-x^*_j) <0, \mbox{ $\forall x^*_j \in \hat A_\gamma$} \} \neq  \emptyset$ {\bf do }\\ \label{algo:while_line}
\nl \ \ \ \ \ \ Find a dominating point $x^*$ by solving the optimization problem \begin{align*} \label{eq:opt_ite}
    x^* =\arg \min_{x} &\ \  (x-\lambda)^T \Sigma^{-1}(x-\lambda) \ \ \ \\
\text{s.t.}\ \ \  &\hat g(x) \geq \hat \kappa,\ \  \\ &(x^*_j-\lambda)^T\Sigma^{-1}(x-x^*_j) <0, \ \mbox{$\forall x^*_j \in\hat A_\gamma$}
\end{align*}

\ \ \ \ and update $\hat A_\gamma \leftarrow\hat A_\gamma \cup \{x^*\}$;\\
\nl {\bf End}\\

\nl Sample $X_1,...,X_{n_2}$ from the mixture distribution $ \sum_{a\in\hat A_\gamma} (1/|\hat A_\gamma|) N(a,\Sigma) $.\\

\nl Compute the IS estimator $\hat\mu_n=( 1/n_2) \sum_{i =1}^ {n_2 } L(X_i) I(X_i \in \bar{\mathcal{S}}_{\gamma}^{\hat{\kappa}})$, where the likelihood ratio $L(X_i)=\phi (X_i;\lambda,\Sigma)/( \sum_{a\in \hat A_\gamma} (1/|\hat A_\gamma|) \phi (X_i;a,\Sigma))$ and $\phi(\cdot;\alpha,\Sigma)$ denotes the density of $N(\alpha,\Sigma)$.

    \caption{{\bf Deep-PrAE to estimate $\mu=P(X\in\mathcal S_\gamma)$.} 
    \label{algo:stage1}}
\end{algorithm}

The first property ensures that, even though the learned set can contain errors, the learned rare-event probability is either an upper or lower bound of the truth. This requirement is important as it is difficult to translate the impact of generalization errors into rare-event estimation errors. By Theorem \ref{counterexample}, we know that any non-zero error implies the risk of missing out on important regions of the rare-event set, undetectably. The one-sided generalization error allows a shift of our target to valid upper and lower bounds that can be correctly estimated, which is the core novelty of Deep-PrAE. 

To this end, we introduce a new efficiency notion:

\begin{definition}
We say an estimator $\hat\mu_n$ satisfies an upper-bound relaxed efficiency certificate to estimate $\mu$ if
$P(\hat\mu_n-\mu<-\epsilon\mu)\leq\delta$
with $n\geq\tilde O(\log(1/\mu))$, for given $0<\epsilon,\delta<1$. \label{relaxed certificate}
\end{definition}
Compared with the efficiency certificate in \eqref{efficiency certificate}, Definition \ref{relaxed certificate} is relaxed to only requiring $\hat\mu_n$ to be an upper bound of $\mu$, up to an error of $\epsilon\mu$. An analogous lower-bound relaxed efficiency certificate can be seen in Appendix \ref{app:lowerbound}. From a risk quantification viewpoint, the upper bound for $\mu$ is more crucial, and the lower bound serves to assess an estimation gap. The following provides a handy certification:

\begin{proposition}[Achieving relaxed efficiency certificate]
Suppose $\hat\mu_n$ is upward biased, i.e., $\overline\mu:=E[\hat\mu_n]\geq\mu$. Moreover, suppose $\hat\mu_n$ takes the form of an average of $n$ i.i.d. simulation runs $Z_i$, with $RE=Var(Z_i)/\overline\mu^2=\tilde O(\log(1/\overline\mu))$. Then $\hat\mu_n$ possesses the upper-bound relaxed efficiency certificate. \label{certificate prop simple}
\end{proposition}

Proposition \ref{certificate prop simple} stipulates that a relaxed efficiency certificate can be attained by an upward biased estimator that has a logarithmic relative error with respect to the biased mean. Appendix \ref{app:conservativeness}
shows an extension of Proposition \ref{certificate prop simple} to two-stage procedures, where the first stage determines the upward biased mean. This upward biased mean, in turn, can be obtained by learning an outer approximation for the rare-event set, giving:

\begin{corollary}[Set-learning + IS]
Consider estimating $\mu=P(X\in\mathcal S_\gamma)$. Suppose we can learn a set $\overline{\mathcal S}_\gamma$ with any number $n_1$ of i.i.d. samples $D_{n_1}$ (drawn from some distribution) such that $\overline{\mathcal S}_\gamma\supset\mathcal S_\gamma$ with probability 1. Also suppose that there is an efficiency certificate for an IS estimator for $\overline\mu(D_{n_1}):=P(X\in\overline{\mathcal S}_\gamma)$. Then a two-stage estimator where a constant $n_1$ number of samples $D_{n_1}$ are first used to construct $\overline{\mathcal S}_\gamma$, and $n_2=\tilde O(\log(1/\overline\mu(D_{n_1}))$ samples are used for the IS in the second stage, achieves the upper-bound relaxed efficiency certificate.\label{relaxed prob}
\end{corollary}

To execute the procedure in Corollary \ref{relaxed prob}, we need to learn an outer approximation of the rare-event set. To this end, consider set learning as a classification problem. Suppose we have collected $n_1$ Stage 1 samples $\{(\tilde X_i, Y_i) \}_{i=1,\ldots,n_1}$, where $Y_i=1$ if $\tilde X_i$ is in the rare-event set $\mathcal S_\gamma$, and 0 otherwise. Here, it is beneficial to use Stage 1 samples that have sufficient presence in $\mathcal S_\gamma$, which can be achieved via any black-box variance reduction methods.
We then consider the pairs $\{(\tilde X_i, Y_i)\}$ where $\tilde X_i$ is regarded as the feature and $Y_i$ as the binary label, and construct a classifier, say $\hat g(x):\mathbb R^d\to[0,1]$, from some hypothesis class $\mathcal G$ that (nominally) signifies $P(Y=1|X=x)$. The learned rare-event set $\overline{\mathcal S}_\gamma$ is taken to be $\{x:\hat g(x)\geq\kappa\}$ for some threshold $\kappa\in\mathbb R$.

The outer approximation requirement $\overline{\mathcal S}_\gamma\supset\mathcal S_\gamma$ means that all true positive (i.e., 1) labels must be correctly classified, or in other words, the false negative (i.e., 0) rate is zero, i.e.,
\begin{equation}
P(X\in\overline{\mathcal S}_\gamma^c,Y=1)=0
\label{zero mis}
\end{equation}
Typically, achieving such a zero ``Type I'' misclassification rate is impossible for any finite sample except in degenerate cases. However, this is achievable under a geometric premise on the rare-event set $\mathcal S_\gamma$ that we call \emph{orthogonal monotonicity}. To facilitate discussion, suppose from now on that the rare-event set is known to lie entirely in the positive quadrant $\mathbb R_+^d$, so in learning the set, we only consider sampling points in $\mathbb R_+^d$ (analogous development can be extended to the entire space).

\begin{definition}
We call a set $\mathcal S\subset\mathbb R_+^d$ \emph{orthogonally monotone} if for any two points $x,x'\in\mathbb R_+^d$, we have $x\leq x'$ (where the inequality is defined coordinate-wise) and $x\in\mathcal S$ implies $x'\in\mathcal S$ too. \label{OM def}
\end{definition}

Definition \ref{OM def} means that any point that is more ``extreme'' than a point in the rare-event set must also lie inside the same set. This is an intuitive assumption that appears to hold in some safety-critical rare-event settings (see Section \ref{sec:numerics}). Note that, even with such a monotonicity property, the boundary of the rare-event set can still be very complex. The key is that, with orthogonal monotonicity, we can now produce a classification procedure that satisfies \eqref{zero mis}. In fact, the simplest approach is to use what we call an orthogonally monotone hull:

\begin{definition}
For a set of points $D=\{x_1,\ldots,x_n\}\subset\mathbb R_+^d$, we define the \emph{orthogonally monotone hull} of $D$ (with respect to the origin) as $\mathcal H(D)=\cup_i\mathcal R(x_i)$, where $\mathcal R(x_i)$ is the rectangle that contains both $x_i$ and the origin as two of its corners.
\label{OM hull def}
\end{definition}
In other words, the orthogonally monotone hull consists of the union of all the rectangles each wrapping each point $x_i$ and the origin $0$. Now, denote $T_0=\{\tilde X_i:Y_i=0\}$ as the non-rare-event sampled points. Evidently, if $\mathcal S_\gamma$ is orthogonally monotone, then $\mathcal H(T_0)\subset\mathcal S_\gamma^c$ (where complement is with respect to $\mathbb R_+^d$), or equivalently, $\mathcal H(T_0)^c\supset\mathcal S_\gamma$, i.e., $\mathcal H(T_0)^c$ is an outer approximation of the rare-event set $\mathcal S_\gamma$. Figure \ref{fig:illustration}(b) shows this outer approximation (and also the inner counterpart). Moreover, $\mathcal H(T_0)^c$ is the smallest region (in terms of set volume) such that \eqref{zero mis} holds, because any smaller region could exclude a point that has label 1 with positive probability.

\textbf{Lazy-Learner IS. }We now consider an estimator for $\mu$ where, given the $n_1$ samples in Stage 1, we build the mixture IS depicted in Theorem 1 to estimate $P(X\in\mathcal H(T_0)^c)$ in Stage 2. Since $\mathcal H(T_0)^c$ takes the form $(\cup_{i:Y_i=0}\mathcal R(\tilde X_i))^c$, it has a finite number of dominating points, which can be found by a sequential algorithm (similar to the one that we will discuss momentarily). We call this the ``lazy-learner'' approach. Its problem, however, is that $\mathcal H(T_0)^c$ tends to have a very rough boundary. This generates a large number of dominating points, many of which are unnecessary in that they do not correspond to any ``true'' dominating points in the original rare-event set $\mathcal S_\gamma$ (see the middle of Figure \ref{fig:illustration}(d)). This in turn leads to a large number of mixture components that degrades the IS efficiency, as the RE bound in Theorem 1 scales linearly with the number of mixture components.

\textbf{Deep-Learning-Based IS. }Our main approach is a deep-learning alternative that resolves the statistical degradation of the lazy learner. We train a neural network classifier, say $\hat g$, using all the Stage 1 samples $\{(\tilde X_i,Y_i)\}$, and obtain an approximate non-rare-event region $(\overline{\mathcal S}_\gamma^\kappa)^c=\{x:\hat g(x)<\kappa\}$, where $\kappa$ is say $1/2$. Then we adjust $\kappa$ minimally away from $1/2$, say to $\hat\kappa$, so that $(\overline{\mathcal  S}_\gamma^{\hat\kappa})^c\subset\mathcal H(T_0)$, i.e., $\hat\kappa=\max\{\kappa\in\mathbb R: (\overline{\mathcal  S}_\gamma^{\kappa})^c\subset\mathcal H(T_0)\}$. Then $\overline{\mathcal  S}_\gamma^{\hat\kappa}\supset\mathcal H(T_ 0)^c\supset\mathcal S_\gamma$, so that $\overline{\mathcal  S}_\gamma^{\hat\kappa}$ is an outer approximation for $\mathcal S_\gamma$ (see Figure \ref{fig:illustration}(c), where $\hat\kappa=0.68$). Stage 1 in Algorithm \ref{algo:stage1} shows this procedure. With this, we can run mixture IS to estimate $P(X\in\overline{\mathcal  S}_\gamma^{\hat\kappa})$ in Stage 2. 

The execution of this algorithm requires the set  $\overline{\mathcal S}_\gamma^{\hat\kappa}=\{x:\hat g(x)\geq\hat\kappa\}$ to be in a form susceptible to Theorem \ref{general IS simplified} and the search of all its dominating points. When $\hat g$ is a ReLU-activated neural net, the boundary of $\hat g(x)\geq\hat\kappa$ is piecewise linear and $\overline{\mathcal S}_\gamma^{\hat\kappa}$ is a union of polytopes, and Theorem \ref{general IS simplified} applies. Finding all dominating points is done by a sequential ``cutting-plane'' method that iteratively locates the next dominating point by minimizing $(x-\mu)^T\Sigma^{-1}(x-\mu)$ over the remaining portion of $\overline{\mathcal S}_\gamma^{\hat\kappa}$ not covered by the local region of any previously found points $x_j^*$. These optimization sequences can be solved via mixed integer program (MIP) formulations for ReLU networks (\citealt{tjeng2017evaluating,huang2018designing}; see Appendix \ref{app:mip}). Note that a user can control the size of these MIPs via the neural net architecture. Regardless of the expressiveness of these networks, Algorithm \ref{algo:stage1} enjoys the following guarantee:

\begin{theorem}[Relaxed efficiency certificate for deep-learning-based mixture IS]
Suppose $\mathcal S_\gamma$ is orthogonally monotone, and $\overline{\mathcal S}_\gamma^{\hat\kappa}$ satisfies the same conditions for $\mathcal S_{\gamma}$ in Theorem \ref{general IS simplified}. Then Algorithm \ref{algo:stage1} attains the upper-bound relaxed efficiency certificate by using a constant number of Stage 1 samples.\label{NN main}
\end{theorem}

Figure \ref{fig:illustration}(d) shows how our deep-learning-based IS achieves superior efficiency compared to other alternatives. The cross-entropy method can miss a dominating point (1) and result in systematic under-estimation. The lazy-learner IS, on the other hand, generates too many dominating points (64) and degrades efficiency. Algorithm \ref{algo:stage1} finds the right number (2) and approximate locations of the dominating points.

Moreover, whereas the upper-bound certificate is guaranteed in our design, in practice, the deep-learning-based IS also appears to work well in controlling the conservativeness of the bound, as dictated by the false positive rate  $P(X\in \overline{\mathcal S}_\gamma^{\hat\kappa},Y=0)$ (see our experiments next). We close this section with a finite-sample bound on the false positive rate. Here, in deriving our bound, we assume the use of a sampling distribution $q$ in generating independent Stage 1 samples, and we use empirical risk minimization (ERM) to train $\hat g$, i.e., $\hat g:=\text{argmin}_{g\in\mathcal{G}}\{ {R}_{n_1}(g):=\frac{1}{n_1}\sum_{i=1}^{n_1}\ell(g(\tilde X_{i}),Y_{i})\}$
where $\ell$ is a loss function and $\mathcal G$ is the considered hypothesis class. Correspondingly, let $R(g):=E_{X\sim q}\ell(g(X),I(X\in\mathcal S_\gamma))$
be the true risk function and $g^{*}:=\arg\min_{g\in\mathcal{G}}R(g)$ its minimizer.
Also let $\kappa^{*}:=\min_{x\in\mathcal{S}_{\gamma}}g^{*}(x)$ be
the true threshold associated with $g^{*}$ in obtaining the smallest outer rare-event set approximation. 

\begin{theorem}[Conservativeness]\label{thm: set_ERM} 
Consider $\overline{\mathcal S}_\gamma^{\hat\kappa}$ obtained in Algorithm \ref{algo:stage1} where $\hat g$ is trained from an ERM. 
Suppose  the density $q$ has bounded 
support $K\subset[0,M]^{d}$ and $0<q_{l}\leq q(x)\leq q_{u}$ for any $x\in K$. Also suppose there exists
a function $h$ such that for any $g\in\mathcal{G}$, $g(x)\geq\kappa$ implies $\ell(g(x),0)\geq h(\kappa)>0$.
(e.g., if $\ell$ is the squared loss, then $h(\kappa)$ could be chosen 
as $h(\kappa)=\kappa^{2}$). Then, with probability at least $1-\delta$,
\begin{align*}
P_{X\sim q} &\left(X\in\bar{\mathcal{S}}_{\gamma}^{\hat{\kappa}}\setminus\mathcal{S}_{\gamma}\right)
\leq \\ 
&\frac{R(g^{*})+2\sup_{g\in\mathcal{G}}\left|R_{n_1}(g)-R(g)\right|}{h(\kappa^{*}-t(\delta,n_1)\sqrt{d} \text{Lip}(g^{*})-\left\Vert \hat{g}-g^{*}\right\Vert _{\infty})}.
\end{align*}
Here, $\text{Lip}(g^{*})$ is the Lipschitz parameter of 
$g^{*}$, and $t(\delta,n_1)=3\left(\frac{\log(n_1q_{l})+d\log M+\log\frac{1}{\delta}}{n_1q_{l}}\right)^{\frac{1}{d}}$.
\end{theorem}

Theorem \ref{thm: set_ERM} reveals a tradeoff between overfitting (measured by $\sup_{g\in\mathcal{G}}\left|R_{n_1}(g)-R(g)\right|$ and $\left\Vert \hat{g}-g^{*}\right\Vert _{\infty}$) and underfitting (measured by $R(g^*)=\inf_{g\in\mathcal{G}}R(g)$).  Appendix \ref{app:conservativeness} discusses related results on the sharp estimates of these quantities for deep neural networks, a more sophisticated version of Theorem \ref{thm: set_ERM} that applies to the cross-entropy loss, a corresponding bound for the lazy learner, as well as results to interpret Theorem \ref{thm: set_ERM} under the original distribution $p$.

Finally, we point out some works in the literature that approximate the Pareto frontier of a monotone function \citep{Wu2018EfficientlyAT,legriel2010approximating}. While the  boundary of an orthogonally monotone set looks similar to the Pareto frontier, our recipe (outer/inner approximation using piecewise-linear-activation NN) is designed to minimize the number of dominating points while simultaneously achieve the relaxed efficiency certificate for rare-event estimation. Such a guarantee is novelly beyond these previous works. 

\section{Numerical Experiments}\label{sec:numerics}

We implement and compare the estimated probabilities and the REs of deep-learning-based IS for the upper bound (Deep-PrAE UB) and lazy-learner IS (LL UB). We also show the corresponding lower-bound estimator (Deep-PrAE LB and LL LB) and benchmark with the cross entropy method (CE), adaptive multilevel splitting (AMS), and naive Monte Carlo (NMC). For CE, we run a few variations testing different parametric classes and report two in the following figures: CE that uses a single Gaussian distribution (CE Naive), representing an overly-simplified CE implementation, and CE with Gaussian Mixture Model with $k$ components (CE GMM-$k$), representing a more sophisticated CE implementation. For Deep-PrAE, we use the samples from CE Naive as the Stage 1 samples. We also run a modification of Deep-PrAE (Deep-PrAE Mod) that replaces $\overline{\mathcal S}_\gamma^{\hat\kappa}$ by $\mathcal S_\gamma$ in the last step of Algorithm \ref{algo:stage1} as an additional comparison. We use a 2-dimensional numerical example and a safety-testing of an intelligent driving algorithm for a car-following scenario. These two experiments are representative as the former is low-dimensional (visualizable) yet with \textit{extremely} rare events while the latter is moderately high-dimensional, challenging for most of the existing methods.

\paragraph{2D Example. }We estimate $\mu=P(X \in \mathcal S_\gamma)$ where $X \sim N([5,5]^T, 0.25I_{2\times2})$, and $\gamma$ ranging from 1.0 to 2.0.  We use $n=30,000$ (10,000 for Stage 1 and 20,000 for Stage 2), and use 10,000 of the CE samples as our Stage 1 samples. Figure \ref{fig:illustration} illustrates the shape of $\mathcal S_\gamma$, which has two dominating points. This probability is microscopically small (e.g., $\gamma = 1.8$  gives $\mu=4.1\times 10^{-24}$) and serves to investigate our performance in ultra-extreme situations. 

\begin{figure}[h]
  \begin{subfigure}{\textwidth}
  \centering
    \includegraphics[width=\textwidth]{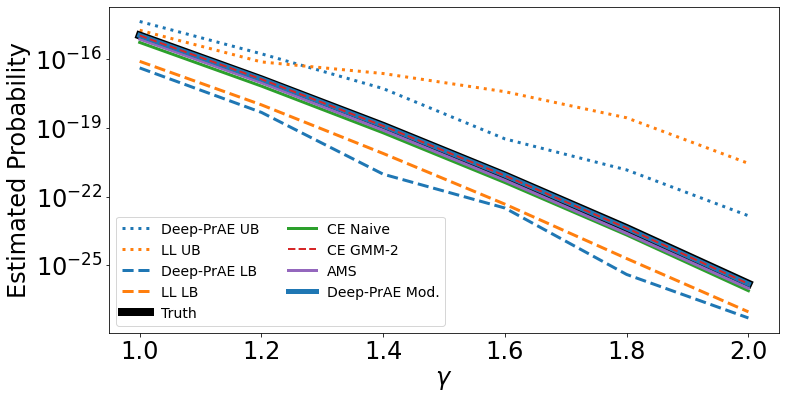}
    \caption{Estimated rare-event probability}
  \end{subfigure}
  
  \begin{subfigure}{\textwidth}
  \centering
    \includegraphics[width=0.95\textwidth]{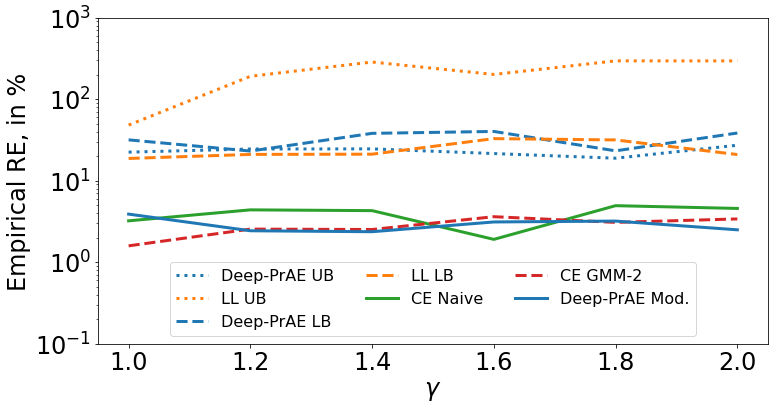}
    \caption{Estimator's empirical relative error}
  \end{subfigure}
  
  \caption{2-dimensional example. Naive Monte Carlo failed in all cases and hence not shown.}
  \label{fig:nonstd_result}
\end{figure}

\begin{figure}[h]
  \centering
    \includegraphics[width=0.95\textwidth]{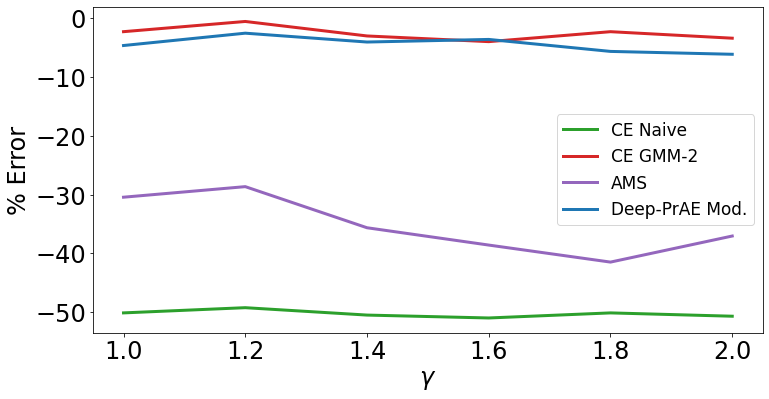}
  \caption{Percentage error of CE, AMS, and modified Deep-PrAE (minus \% error means under-estimation)}
  \label{fig:pe_result}
\end{figure}

Figure \ref{fig:nonstd_result} compares all approaches to the true value, which we compute via a proper mixture IS with 50,000 samples assuming the full knowledge of $\mathcal S_\gamma$. It shows several observations. First, Deep-PrAE and LL (both UB and LB) always provide valid bounds that contain the truth. Second, the UB for LL is more conservative than Deep-PrAE in up to two orders of magnitudes, which is attributed to the overly many (redundant) dominating points. Correspondingly, the RE of LL UB is tremendously high, reaching over $500\%$ when $\gamma = 2.0$, compared to around 40\% for Deep-PrAE UB, Deep-PrAE LB, and LL LB. Third, CE Naive, which finds only one dominating point, consistently under-estimates the truth by about $50\%$, yet it gives an over-confident RE, e.g., $<5\%$ when $\gamma < 2$. This shows a systematic undetected under-estimation issue when CE is implemented overly-naively. AMS also underestimates the true value by 30\%-40\%, while CE GMM-2 and Deep-PrAE Mod perform empirically well.  Figure \ref{fig:pe_result} summarizes the zoomed-in performances of CE Naive, CE GMM-2, AMS, and Deep-PrAE Mod in terms of percentage error, which is the difference between the estimated and true probability as a percentage of the true value. It shows that while CE performs well when the IS parametric class is well-chosen (CE GMM-2), a poor CE parametric class (CE Naive) as well as AMS could under-estimate. Yet our Deep-PrAE, despite using samples from a poor CE class in Stage 1, can recover valid results: Deep-PrAE provides a valid UB, and Deep-PrAE Mod gives an estimate as good as the good CE class.

\paragraph{Intelligent Driving Example.} In this example, we evaluate the crash probability of car-following scenarios involving a human-driven lead vehicle (LV) followed by an autonomous vehicle (AV). The AV is controlled by the Intelligent Driver Model (IDM) to maintain safety distance while ensuring smooth ride and maximum efficiency. IDM model is widely used for autonomy evaluation and microscopic transportation simulations 
\citep{Treiber_2000, dlforcf, rssidm}. The state at time $t$ is given by 6 states consisting of the position, velocity, and acceleration of both LV and AV. The dynamic system has a stochastic input $u_t$ related to the acceleration of the LV and subject to uncertain human behavior. We consider an evaluation horizon $T=60$ seconds and draw a sequence of 15 Gaussian random actions at a 4-second epoch, leading to a 15-dimensional LV action space. A (rare-event) crash occurs at time $t \leq T$ if the longitudinal distance $r_t$ between the two vehicles is negative, with $\gamma$ parameterizing the AV maximum throttle and brake pedals. This rare-event set is analytically challenging (see \citealt{zhao2017accelerated} for a similar setting). More details are in Appendix \ref{app:experiment}.

\begin{figure}[h]
  \begin{subfigure}{\textwidth}
  \centering
    \includegraphics[width=\textwidth]{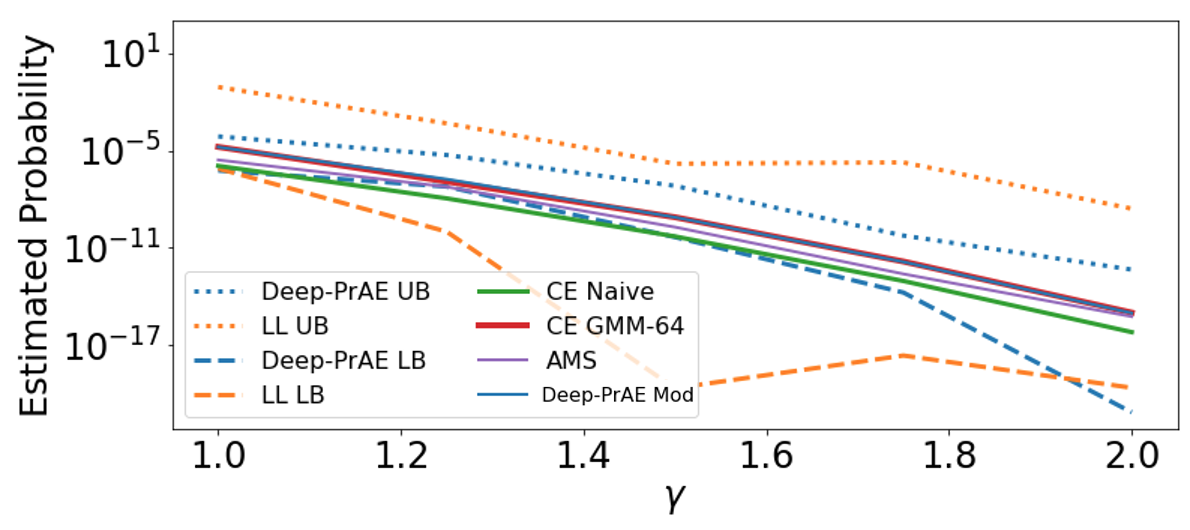}
    \caption{Estimated rare-event probability}
  \end{subfigure}%
  
  \begin{subfigure}{\textwidth}
  \centering
    \includegraphics[width=\textwidth]{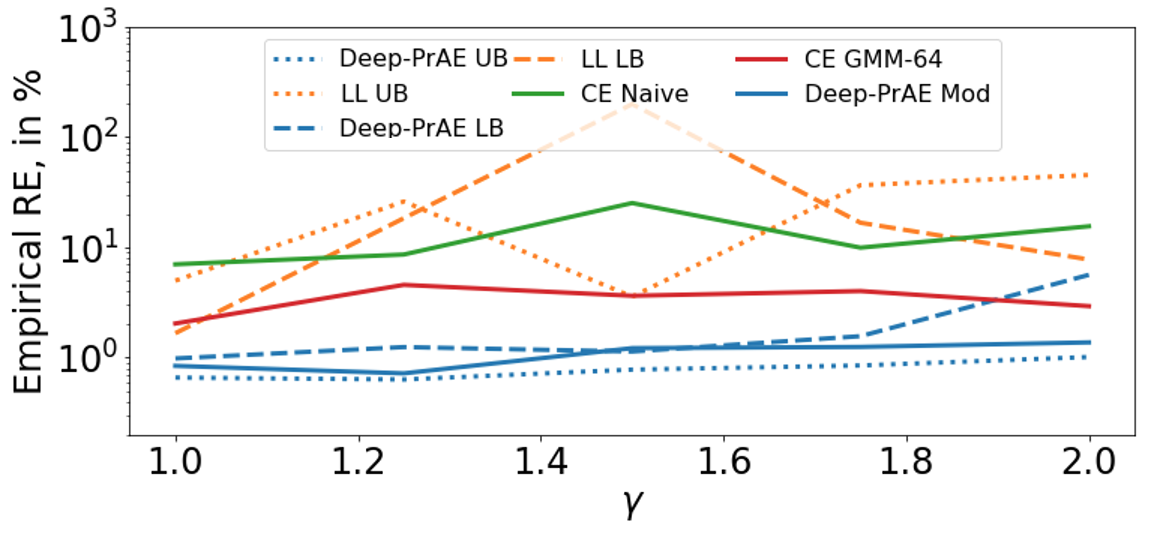}
    \caption{Estimator's empirical relative error}
  \end{subfigure}
  \caption{Intelligent driving example. Naive Monte Carlo failed in all cases and hence not shown.}
  \label{fig:nonstd_result_highdim}
\end{figure}
  
Figure  \ref{fig:nonstd_result_highdim} shows the performances of competing approaches, using $n = 10,000$. For CE, we use a single Gaussian (CE Naive) and a large number of mixtures (CE GMM-64). Deep-PrAE and LL (UB and LB) appear consistent in giving upper and lower bounds for the target probability, and Deep-PrAE produces tighter bounds than LL ($10^{-2}$ vs $10^{-6}$ in general). LL UB has $5,644$ dominating points when $\gamma=1$ vs 42 in Deep-PrAE, and needs 4 times more time to search for them than Deep-PrAE. Moreover, the RE of Deep-PrAE is 3 times lower than LL across the range (in both UB and LB). Thus, Deep-PrAE outperforms LL in both tightness, computation speed, and RE. CE Naive and AMS seem to give a stable estimation, but evidence shows they are under-estimated: Deep-PrAE Mod and CE GMM-64 lack efficiency certificates and thus could under-estimate, and the fact that their estimates are higher than CE Naive and AMS suggests both CE Naive and AMS are under-estimated. Lastly, NMC fails to give a single hit in all cases (thus not shown on the graphs). Thus, among all approaches, Deep-PrAE stands out as giving reliable and tight bounds with low RE.

\paragraph{Summary of Practical Benefits. }Our investigation shows strong evidence of the practical benefits of using Deep-PrAE for rare-event estimation. It generates valid bounds for the target probability with low RE and improved efficiency. The use of classifier prediction helps reduce the computational effort from running more simulations. For example, to assess whether the AV crash rate is below $10^{-8}$ for $\gamma = 1.0$, only 1000 simulation runs would be needed by Deep-PrAE UB or LB to get around $1\%$ RE, which takes about $400$ seconds in total. This is in contrast to 3.7 months for naive Monte Carlo.

\section{Discussion and Future Work}

In this paper, we proposed a robust certifiable approach to estimate rare-event probabilities in safety-critical applications. The proposed approach designs efficient IS distribution by combining the dominating point machinery with deep-learning-based rare-event set learning. We study the theoretical guarantees and present numerical examples. The key property that distinguishes our approach with existing black-box rare-event simulation methods is our correctness guarantee. Leveraging on a new notion of relaxed efficiency certificate and the orthogonal monotonicity assumption, our approach avoids the perils of undetected under-estimation as potentially encountered by other methods.

We discuss some key assumptions in our approach and related prospective follow-up works. First, the orthogonal monotonicity assumption appears an important first step to give new theories on black-box rare-event estimation beyond the existing literature. Indeed, we show that even with this assumption, black-box approaches such as CE and splitting can suffer from the dangerous pitfall of undiagnosed under-estimation, and our approach corrects for it. The real-world values of our approach are: (1) We rigorously show why our method has better performances in the orthogonally monotone cases; (2) For tasks close to being orthogonal monotonic (e.g., the IDM example), our method is empirically more robust; (3) For non-orthogonally-monotone tasks, though directly using our approach does not provide guarantees, we could potentially train mappings to latent spaces that are orthogonally monotone. We believe such type of geometric assumptions comprises a key ingredient towards a rigorous theory for black-box rare-event estimation that warrants much further developments.

Second, the dominating point search algorithm in our approach assumes Gaussian randomness. To this end we can relax it in two directions: by fitting a GMM with a sufficiently large number of components, or using light-tailed distributions (i.e., with finite exponential moments), since the dominating point machinery applies. These extensions will be left for future work.

Finally, the tightness of the upper bound depends on the sample quality. An ideal method in Stage 1 would generate samples close to the rare-event boundary to produce good approximations. Cutting the Stage 1 effort by, e.g., designing iterative schemes between Stages 1 and 2, will also be a topic for future investigation.

\section*{Acknowledgments}
We gratefully acknowledge support from the National Science Foundation under grants CAREER CMMI-1834710,  IIS-1849280 and IIS-1849304. Mansur Arief and Wenhao Ding are supported in part by Bosch.

\bibliography{references}

\newpage
\onecolumn

\appendix
\appendixpage

We present supplemental results and discussions. Appendix \ref{append:challenge} expands Section \ref{sec:existing} regarding Monte Carlo efficiency and variance reduction. Appendix \ref{app:mip} provides further details on Algorithm \ref{algo:stage1}, in particular the mixed integer formulation used to solve the underlying optimization problems. Appendix \ref{app:conservativeness} expands the efficiency and conservativeness results in Section \ref{sec:Deep-PrAE}. Appendix \ref{app:lowerbound} presents the lower-bound relaxed efficiency certificate and estimators in parallel to the upper-bound results in Section \ref{sec:Deep-PrAE}. Appendix \ref{app:ce_ams} provides an overview of the cross-entropy method and multi-level splitting (or subset simulation) and discusses their perils for black-box problems. Appendix \ref{app:experiment} illustrates further experimental results. Finally, Appendix \ref{app:proofs} shows all technical proofs.

\section{Further Details for Section~\ref{sec:existing}}
\label{append:challenge}
This section expands the discussions in Section \ref{sec:existing}, by explaining in more detail the notion of relative error, challenges in naive Monte Carlo, the concept of dominating points, and the perils of black-box variance reduction algorithms. 

\subsection{Explanation of the Role of Relative Error}
As described in Section \ref{sec:existing}, to estimate a tiny $\mu$ using $\hat\mu_n$, we want to ensure a high accuracy in relative term, namely, \eqref{efficiency certificate}. Suppose that $\hat\mu_n$ is unbiased and is an average of $n$ i.i.d. simulation runs, i.e., $\hat\mu_n=(1/n)\sum_{i=1}^n Z_i$ for some random unbiased output $Z_i$. The Markov inequality gives that
$$
P(|\hat\mu_n-\mu|>\epsilon\mu)\leq\frac{Var(\hat\mu_n)}{\epsilon^2\mu^2}=\frac{Var(Z_i)}{n\epsilon^2\mu^2}$$
so that
$$\frac{Var(Z_i)}{n\epsilon^2\mu^2}\leq\delta$$
ensures \eqref{efficiency certificate}. Equivalently, 
$$n\geq\frac{Var(Z_i)}{\delta\epsilon^2\mu^2}=\frac{RE}{\delta\epsilon^2}$$
is a sufficient condition to achieve \eqref{efficiency certificate}, where $RE=Var(Z_i)/\mu^2$ is the relative error defined as the ratio of variance (per-run) and squared mean.

Note that replacing the second $\mu$ with $\hat\mu_n$ in the left hand side of \eqref{efficiency certificate} does not change the condition fundamentally, as either is equivalent to saying the ratio $\hat{\mu}_n/\mu$ should be close to 1. Also, note that we focus on the nontrivial case that the target probability $\mu$ is non-zero but tiny; if $\mu=0$, then no non-zero Monte Carlo estimator can achieve a good relative error. 

\subsection{Further Explanation on the Challenges in Naive Monte Carlo}
We have seen in Section \ref{sec:existing} that for the naive Monte Carlo estimator, where $Z_i=I(X_i\in\mathcal S_\gamma)$, the relative error is $RE=\mu(1-\mu)/\mu^2=(1-\mu)/\mu$. Thus, when $\mu$ is tiny, the sufficient condition for $n$ to attain \eqref{efficiency certificate} scales at least linearly in $1/\mu$. In fact, this result can be seen to be tight by analyzing $n\hat\mu_n$ as a binomial variable. To be more specific, we know that $P(|\hat{\mu}_n-\mu|>\varepsilon\mu)=P(|n\hat{\mu}_n-n\mu|>\varepsilon n\mu)$ and that $n\hat{\mu}_n$ takes values in $\{0,1,\dots,n\}$. Therefore, if $n\mu\to0$, then $P(|\hat{\mu}_n-\mu|>\varepsilon\mu)\to1$, and hence \eqref{efficiency certificate} does not hold.

Moreover, the following provides a concrete general statement that an $n$  that grows only polynomially in $\gamma$ would fail to estimate $\mu$ that decays exponentially in $\gamma$ with enough relative accuracy, of which \eqref{efficiency certificate} fails to hold is an implication. 
\begin{proposition}
Suppose that $\mu=P(X\in\mathcal{S}_{\gamma})$ is exponentially decaying in $\gamma$ and $n$ is polynomially growing in $\gamma$. Define $\hat{\mu}_n=(1/n)\sum_{i=1}^n I(X_i\in \mathcal{S}_{\gamma})$. Then for any $0<\varepsilon<1$, 
$$
\lim_{\gamma\rightarrow\infty}P(|\hat{\mu}_n-\mu|>\varepsilon\mu)=1.
$$
\label{naive Monte Carlo}
\end{proposition}

We have used the term efficiency certificate to denote an estimator that achieves \eqref{efficiency certificate} with $n=\tilde{O}(\log(1/\mu))$. In the rare-event literature, such an estimator is known as ``logarithmically efficient" or ``weakly efficient" \citep{juneja2006rare,blanchet2012state}. 

\subsection{Further Explanations of Dominating Points }
We have mentioned that a certifiable IS should account for all dominating points, defined in Definition \ref{dominating point def}. We provide more detailed explanations here. Roughly speaking, for $X\sim N(\lambda,\Sigma)$ and a rare-event set $\mathcal S_\gamma$, the Laplace approximation gives $P(X\in\mathcal S_\gamma)\approx e^{-\inf_{a\in\mathcal S_\gamma}\frac{1}{2}(a-\lambda)^T\Sigma^{-1}(a-\lambda)}$ (see the proof of Theorem \ref{general IS}).
Thus, to obtain an efficiency certificate, IS estimator given by $Z=L(X)I(X\in\mathcal S_\gamma)$, where $X\sim\tilde p$ and $L=dp/d\tilde p$, needs to have $\widetilde{Var}(Z)\leq\tilde E[Z^2]\approx e^{-\inf_{a\in\mathcal S_\gamma}(a-\lambda)^T\Sigma^{-1}(a-\lambda)}$ (where $\widetilde{Var}(\cdot)$ and $\tilde E[\cdot]$ denote the variance and expectation under $\tilde p$, and $\approx$ is up to some factor polynomial in $\inf_{a\in\mathcal S_\gamma}(a-\lambda)^T\Sigma^{-1}(a-\lambda)$; note that the last equality relation cannot be improved, as otherwise it would imply that $\widetilde{Var}(Z)=\tilde E[Z^2]-(\tilde E[Z])^2<0$). 

Now consider an IS that translates the mean of the distribution from $\mu$ to $a^*=\mbox{argmin}_{a\in\mathcal S_\gamma}(a-\lambda)^T\Sigma^{-1}(a-\lambda)$, an intuitive choice since $a^*$ contributes the highest density among all points in $\mathcal S_\gamma$ (this mean translation also bears the natural interpretation as an exponential change of measure; \citep{bucklew2013introduction}). The likelihood ratio is $L(x)=e^{(\mu-a^*)^T\Sigma^{-1}(x-\lambda)+\frac{1}{2}(\lambda-a^*)^T\Sigma^{-1}(\lambda-a^*)}$, giving
\begin{equation}\label{second moment}
\tilde E[Z^2]=\tilde E[L(X)^2I(X\in\mathcal S_\gamma)]=e^{-(a^*-\lambda)^T\Sigma^{-1}(a^*-\lambda)}\tilde E[e^{-2(a^*-\lambda)^T\Sigma^{-1}(x-a^*)}I(X\in\mathcal S_\gamma)]
\end{equation}
If the ``overshoot''  $(a^*-\lambda)^T\Sigma^{-1}(x-a^*)$, i.e., the remaining term in the exponent of $L(x)$ after moving out $-(a^*-\lambda)^T\Sigma^{-1}(a^*-\lambda)$,  satisfies $(a^*-\lambda)^T\Sigma^{-1}(x-a^*)\geq0$ for all $x\in\mathcal S_\gamma$, then the expectation in the right hand side of (\ref{second moment}) is bounded by 1, and an efficiency certificate is achieved. This, however, is not true for all set $\mathcal S_\gamma$, which motivates the following definition of the dominant set and points in Definition \ref{dominating point def}.

For instance, if $\mathcal S_\gamma$  is convex, then, noting that $(x-\lambda)^T\Sigma^{-1}$ is precisely the gradient of the function $(1/2)(x-\lambda)^T\Sigma^{-1}(x-\lambda)$, we get that $a^*$ gives a singleton dominant set since  $(a^*-\lambda)^T\Sigma^{-1}(x-a^*)\geq0$ for all $x\in\mathcal S_\gamma$ is precisely the first order optimality condition of the involved quadratic optimization. In general, if we can decompose 
$\mathcal S_\gamma=\bigcup_j\mathcal S_\gamma^j$ where $\mathcal S_\gamma^j=\{x:(a_j-\lambda)^T\Sigma^{-1}(x-a_j)\geq0\}$ for a dominating point $a_j\in A_\gamma$, then each $\mathcal S_\gamma^j$ can be viewed as a ``local'' region where the dominating point $a_j$ is the highest-density, or the most likely point such that the rare event occurs.

The following is the detailed version of Theorem \ref{general IS simplified}:

\begin{theorem}[Certifiable IS]
Suppose that $A_{\gamma}$ is the dominant set for $\mathcal S_\gamma$ associated with the distribution $N(\lambda,\Sigma)$. Then we can decompose $\mathcal S_\gamma=\bigcup_j \mathcal S_\gamma^j$ where $\mathcal S_\gamma^j$'s are disjoint, $a_j\in \mathcal S_\gamma^j$ and $\mathcal S_\gamma^j\subset\{x:(a_j-\lambda)^T\Sigma^{-1}(x-a_j)\geq 0\}$ for $a_j\in A_{\gamma}$. Denote $a^*=\arg\min\{(a_j-\lambda)^T\Sigma^{-1}(a_j-\lambda):a_j\in A_{\gamma}\}$. Assume that each component of $a^*$ is of polynomial growth in $\gamma$. Moreover, assume that there exist invertible matrix $B$ and positive constant $\varepsilon$ such that $\{x:B(x-a^*)\geq 0,(x-a^*)^T\Sigma^{-1}(x-a^*)\leq \varepsilon^2\}\subset \mathcal S_\gamma$. Then the IS distribution $\sum_j \alpha_j N(a_j,\Sigma)$ achieves an efficiency certificate in estimating $\mu=P(X\in S_\gamma)$, i.e., if we let $Z=I(X\in \mathcal S_\gamma)L(X)$ where $L$ is the corresponding likelihood ratio, then $\tilde{E}[Z^2]/\tilde{E}[Z]^2$ is at most polynomially growing in $\gamma$. This applies in particular to $\mathcal S_\gamma=\{x:f(x)\geq \gamma\}$ where $f(x)$ is a piecewise linear function.\label{general IS}

\end{theorem}

We contrast Theorem \ref{general IS} with existing works on dominating points. The latter machinery has been studied in \citep{sadowsky1990large,dieker2006fast}. These papers, however, consider regimes where the G\"artner-Ellis Theorem \citep{Gartner1977,Ellis1984} can be applied, which requires the considered rare-event set to scale proportionately with the rarity parameter. This is in contrast to the general conditions on the dominating points used in Theorem \ref{general IS}.

\subsection{Further Explanation of the Example in Theorem \ref{counterexample}}
In the theorem, there are two dominating points $\gamma$ and $-k\gamma$ but the IS design only considers the first one. As a result, there could exist ``unlucky'' scenario where the sample falls into the rare-event set, so that $I(X\in\mathcal S_\gamma)=1$, while the likelihood ratio $L(X)$ explodes, which leads to a tremendous estimation variance. Part 2 of the theorem further shows how this issue is undetected empirically, as the empirical RE appears small (polynomially in $n$ and hence $\gamma$ by our choice of $n$) while the estimation concentrates at a value that can be severely under the correct one (especially when $k<1$). This is because the samples all land on the neighborhood of the solely considered dominating point. If the missed dominating point is a significant contributor to the rare-event probability, then the empirical performance would look as if the rare-event set is \emph{smaller}, leading to a systematic under-estimation. Note that this phenomenon occurs even if the estimator is unbiased, which is guaranteed by IS by default.

\section{Further Details on Implementing Algorithm \ref{algo:stage1}}
\label{app:mip}

We provide further details on implementing Algorithm \ref{algo:stage1}. In particular, we present how to solve the optimization problem
\begin{align} 
    x^*=\arg \min_{x} &\ \  (x-\lambda)^T \Sigma^{-1}(x-\lambda) \ \ \ 
\text{s.t.}\ \ \  \hat g(x) \geq \hat \kappa,\ \   (x^*_j-\lambda)^T\Sigma^{-1}(x-x^*_j) <0\ \mbox{$\forall x^*_j \in\hat A_\gamma$}\label{MIP problem}
\end{align}
to obtain the next dominating point in the sequential cutting-plane approach in Stage 2. Moreover, we also present how to tune \begin{equation}
\hat\kappa=\max\{\kappa\in\mathbb R:(\overline{\mathcal S}_\gamma^\kappa)^c\subset\mathcal H(T_0)\}\label{tune kappa}
\end{equation}
in Stage 1.

\paragraph{MIP formulations for ReLU-activated neural net classifier.} 

The problem \eqref{MIP problem} can be reformulated into a mixed integer program (MIP), in the case where $\hat g(x)$ is trained via a ReLU-activated neural net classifier, which is used in our deep-learning-based IS. 
Since the objective is convex quadratic and second set of constraints is linear in \eqref{MIP problem}, we focus on the first constraint $\hat g(x) \geq \gamma$. 
The neural net structure $\hat g (x)$ in our approach (say with $n_g$ layers) can be represented as $\hat g(x)=(\hat g_{n_g}\circ ...\circ \hat g_1) (x)$, where each $\hat g_i (\cdot)$ denotes a ReLU-activated layer with linear transformation, i.e. $\hat g_i (\cdot)=\max\{ LT(\cdot),0 \}$, where $LT(\cdot)$ denotes a certain linear transformation in the input. In order to convert $\hat g(\cdot)$ into an MIP constraint, we introduce $M$ as a practical upper bound for $x_1,...,x_n$ such that $|x_i| < M$. The key step is to reformulate the ReLU function
$y=\max\{x,0 \}$ into \begin{align*}
     & y\leq x + M (1-z) \\
     & y\geq x\\
     & y\leq M z\\
     & y\geq 0\\
     & z \in \{0,1\}.
\end{align*} 
For simple ReLU networks, the size of the resulting MIP formulation depends linearly on the number of neurons in the neural network. In particular, the number of binary decision variables is linearly dependent on the number of ReLU neurons, and the number of constraints is linearly dependent the total number of all neurons (here we consider the linear transformations as independent neurons). 

The MIP reformulation we discussed can be generalized to many other popular piecewise linear structures in deep learning. For instance, linear operation layers, such as normalization and convolutional layers, can be directly used as constraints; some non-linear layers, such as ReLU and max-pooling layers, introduce  non-linearity by the ``max'' functions. A general reformulation for the max functions can be used to convert these non-linear layers to mixed integer constraints.

Consider the following equality defined by a max operation $y=\max\{x_1,x_2,...,x_n\}$. Then the equality is equivalent to \begin{align*}
     & y\leq x_i + 2M (1-z_i), i=1,...,n  \\
     & y\geq x_i, i=1,...,n\\
     & \sum_{i=1,...,n} z_i = 1\\
     & z_i \in \{0,1\}.
\end{align*}

\paragraph{Tuning $\hat\kappa$. } We illustrate how to tune $\hat\kappa$ to achieve \eqref{tune kappa}. This requires checking, for a given $\kappa$, whether $(\overline{\mathcal  S}_\gamma^{\kappa})^c\subset\mathcal H(T_0)$. Then, by discretizing the range of $\kappa$ or using a bisection algorithm, we can leverage this check to obtain \eqref{tune kappa}.

We use an MIP to check  $(\overline{\mathcal  S}_\gamma^{\kappa})^c\subset\mathcal H(T_0)$. Recall that $\mathcal H(T_0)=\bigcup_{i:Y_i=0}\{x\in\mathbb R_+^d:x\leq\tilde X_i\}$. We want to check if $\{x\in\mathbb R_+^d:\hat g(x)\leq \kappa \}$ for a given $\kappa$ lies completely inside the hull, where $\hat g(x)$ is trained with a ReLU-activated neural net. This can be done by solving an optimization problem as follows. First, we rewrite $\mathcal H(T_0)$ as $\{x\in\mathbb R_+^d:\min_{i=1,\ldots,n}\max_{j=1,\ldots,d}\{x^j-\tilde X_i^j\}\leq0\}$, where $x^j$ and $x_i^j$ refer to the $j$-th components of $x$ and $\tilde X_i$ respectively. Then we solve
\begin{equation}
\begin{array}{ll}
\max_{x\in\mathbb R^d}&\min_{i=1,\ldots,n}\max_{j=1,\ldots,d}\{x^j-\tilde X_i^j\}\\
\text{subject to}&\hat g(x)\leq \kappa\\
&x\geq0
\end{array}\label{opt}
\end{equation}
If the optimal value is greater than 0, this means $\{x\in\mathbb R_+^d:\hat g(x)\leq \kappa \}$ is not completely inside $\mathcal H(T_0)$, and vice versa. Now, we rewrite \eqref{opt} as

\begin{equation}
\begin{array}{ll}
\max_{x\in\mathbb R^d,\beta\in\mathbb R}&\beta\\
\text{subject to}&\max_{j=1,\ldots,d}\{x^j-\tilde X_i^j\}\geq\beta\ \forall i=1,\ldots,n\\
&\hat g(x)\leq \kappa\\
&x\geq0
\end{array}\label{opt1}
\end{equation}
We then rewrite \eqref{opt1} as an MIP by introducing a large real number $M$ as a practical upper bound for all coordinates of $x$:

\begin{equation}
\begin{array}{ll}
\max_{x\in\mathbb R^d,\beta\in\mathbb R}&\beta\\
\text{subject to}& x^j-\tilde X_i^j + 4M (1-z_{ij})\geq\beta\ \ \  \forall i=1,\ldots,n, j=1,\ldots,d\\
&\sum_{j=1,...,d} z_{ij} \geq 1\ \ \  \forall i=1,\ldots,n\\
&z_{ij}\in \{0,1\}\ \ \  \forall i=1,\ldots,n, j=1,\ldots,d\\
&\hat g(x)\leq \kappa\\
&x\geq0
\end{array}\label{opt2}
\end{equation}

Note that the set of points $T_0$ to be considered in constructing $\mathcal H(T_0)$ can be reduced to its ``extreme points''. More concretely, we call a point $x\in T_0$ an extreme point if there does not exist any other point $x'\in T_0$ such that $x\leq x'$.  We can eliminate all points $x\in T_0$ such that $x\leq x'$ for another $x'\in T_0$, and the resulting orthogonal monotone hull would remain the same. If we carry out this elimination, then in \eqref{opt1} we need only consider $\tilde X_i$ that are extreme points in $\mathcal H(T_0)$, which can reduce the number of integer variables needed to add. In practice, we can also randomly remove points in $T_0$ to further reduce the number of integer variables. This would not affect the correctness of our approach, but would increase the conservativeness of the final estimate.

\section{Further Results for Section \ref{sec:Deep-PrAE}}
\label{app:conservativeness}

Here we present and discuss several additional results for Section \ref{sec:Deep-PrAE} regarding estimation efficiency and conservativeness. The latter includes further theorems on the lazy-learner classifier and classifiers constructed using the difference of two functions, translation of the false positive rate under the Stage 1 sampling distribution to under the original distribution, and interpretations and refinements of the conservativeness results. 

\subsection{Extending Upper-Bound Relaxed Efficiency Certificate to Two-Stage Procedures}
We present an extension of Proposition \ref{certificate prop simple} to two-stage procedures, which is needed to set up Corollary \ref{relaxed prob}.
\begin{proposition}[Extended relaxed efficiency certificate] \label{prop:extend}

Suppose constructing $\hat\mu_n=\hat\mu_{n_2}(D_{n_1})$ consists of two stages, with $n=n_1+n_2$: First we sample $D_{n_1}=\{\tilde X_1,\ldots,\tilde X_{n_1}\}$, where $\tilde X_i$ are i.i.d. (following some sampling distribution), and given $D_{n_1}$, we construct $\hat\mu_{n_2}(D_{n_1})=(1/n_2)\sum_{i=1}^{n_2}Z_i$ where $Z_i$ are i.i.d. conditional on $D_{n_1}$ (following some distribution). Suppose $\hat\mu_n$ is conditionally upward biased almost surely, i.e., $\overline\mu(D_{n_1}):=E[\hat\mu_n|D_{n_1}]\geq\mu$, and the conditional relative error given $D_{n_1}$ in the second stage satisfies $RE(D_{n_1}):=Var(Z_i|D_{n_1})/\overline\mu(D_{n_1})^2=\tilde O(\log(1/\overline\mu(D_{n_1})))$. If $n_1=\tilde O(\log(1/\mu))$ (such as a constant number), then $\hat\mu_n$ possesses the upper-bound relaxed efficiency certificate.
\end{proposition}

\subsection{Conservativeness of Lazy Learner}
We provide a result to quantify the conservativeness of the lazy-learner IS in terms of the false positive rate. Recall that the lazy learner constructs the outer approximation of the rare-event set using $\mathcal{H}({T_0})^c$, which is the complement of the orthogonal monotone hull of the set of all non-rare-event samples. The conservativeness is measured concretely by the set difference between $\mathcal{H}({T_0})^c$ and $\mathcal{S}_{\gamma}$, for which we have the following result:
\begin{theorem}[Conservativeness of lazy learner]
\label{thm: set_OMhull}Suppose that the density $q$ has bounded
support $K\subset[0,M]^{d}$, and $0<q_{l}\leq q(x)\leq q_{u}$ for any $x\in K$. 
Then, with probability at least $1-\delta$,
\begin{align*}
P_{X\sim q}(X\in\mathcal{H}(T_{0})^{c}\backslash\mathcal{S}_{\gamma})&\leq M^{d-1}q_{u}\left(\frac{\sqrt{d}}{2}\right)^{d-1}w_{d-1}t(\delta,n_1)\\
&=\sqrt{\frac{e}{\pi(d-1)}}\left(\frac{1}{2}\pi e\right)^{\frac{d-1}{2}}q_{u}t(\delta,n_1)(1+O(d^{-1})).
\end{align*}
Here  $t(\delta,n_1)=3\left(\frac{\log(n_1q_{l})+d\log{M}+\log\frac{1}{\delta}}{n_1q_{l}}\right)^{\frac{1}{d}}$,
$w_{d}$ is the volume of a $d-$dimensional Euclidean ball of radius
1, and the last $O(\cdot)$ is as $d$ increases. 

\end{theorem}

\subsection{Translating the False Positive Rate to under the Original distribution}\label{subsec: false_positive_p}

Theorems \ref{thm: set_ERM} and \ref{thm: set_OMhull} are stated with respect to $q$, the sampling distribution used in the first stage. We explain how to translate the false positive rate results to under the original distribution $p$. In the discussion below, we will consider Theorem \ref{thm: set_ERM} (and Theorem \ref{thm: set_OMhull} can be handled similarly). In this case, our target is to give an upper bound to $P_{X\sim p}(X\in\mathcal{S}_{\gamma}^{\hat{\kappa}}\backslash\mathcal{S}_{\gamma})$ based on the result of Theorem \ref{thm: set_ERM}.

If the true input distribution $p$ does not have a bounded
support, we can first choose $M$ to be large to make sure that $P_{X\sim p}(X\notin[0,M]^{d})$ is small compared to the probability of $\mathcal S_\gamma$. We argue that we do not need $M$ to be too large here. Indeed, if $p$ is light tail (e.g., a distribution with tail probability exponential in $M$), then the required $M$ grows at most polynomially
in $\gamma$. 

Having selected $M$, and with the freedom in selecting $q$ in Stage 1, we could make sure
that in $[0,M]^d $,  $q(x)$ is bounded away from 0 (e.g., we can choose $q$ to be the uniform distribution over $[0,M]^d$). Then, by Theorem \ref{thm: set_ERM} and a change of measure argument, we can give a bound for $P_{X\sim p}(X\in{[0,M]^d},X\in\mathcal{S}_{\gamma}^{\hat{\kappa}}\backslash\mathcal{S}_{\gamma})$.
Finally, we bound the false positive rate with respect to $p$ by  $P_{X\sim p}(X\in\mathcal{H}(T_{0})^{c}\backslash\mathcal{S}_{\gamma})\leq P_{X\sim p}(X\notin[0,M]^{d})+P_{X\sim p}(X\in{[0,M]^d},X\in\mathcal{H}(T_{0})^{c}\backslash\mathcal{S}_{\gamma})$.

\subsection{Conservativeness Results for Classifiers Constructed Using Differences of Two Trained Functions}\label{subsec: ERM_modified}

Theorem \ref{thm: set_ERM} presents a conservativeness result when $\hat g$ is trained with an empirical risk minimization (ERM). In this subsection, we will show a more sophisticated version of Theorem \ref{thm: set_ERM}, which corresponds more closely to the $\hat g$ that we implemented in our experiments. Suppose that the Stage 1 samples are generated in the same way as in Algorithm \ref{algo:stage1}. We let $\mathcal{F}:=\{f_{\theta}\}$ denote the function class
induced by the model. Here a main difference with previously is that we allow functions in $\mathcal{F}$ to be 2-dimensional, and both the loss function and the classification boundary will be constructed from these 2-dimensional functions. 

Suppose that $f_{\theta}$ is the output a neural
network with 2 neurons in the output layer, and denote them as $f_{\theta,0},f_{\theta,1}$.
Let the loss function evaluated at the $i$-th sample be $\ell(f_{\theta}(\tilde X_{i}),Y_{i})$. 
For example, the cross-entropy loss is given by $-\left[I(Y_{i}=0)\log\frac{e^{f_{\theta,0}(\tilde X_{i})}}{e^{f_{\theta,0}(\tilde X_{i})}+e^{f_{\theta,1}(\tilde X_{i})}}+I(Y_{i}=1)\log\frac{e^{f_{\theta,1}(\tilde X_{i})}}{e^{f_{\theta,0}(\tilde X_{i})}+e^{f_{\theta,1}(\tilde X_{i})}}\right]$.
Like in the ERM approach in Theorem \ref{thm: set_ERM}, we compute
$\hat{f}=f_{\hat{\theta}}\in\mathcal{F}$ which is the minimizer of
the empirical risk, i.e., $\hat{f}=\arg\min_{f_{\theta}\in\mathcal{F}}R_{n_1}(f_{\theta})$.
For each function $f_{\theta}\in\mathcal{F}$, define function $g_{\theta}$ as $g_{\theta}:=f_{\theta,1}-f_{\theta,0}$. In
this modified approach, the learned rare-event set would be given
by $\tilde{\mathcal{S}}_{\gamma}^{\kappa}:=\{x:g_{\hat{\theta}}(x)\geq\kappa\}$,
and to make sure that $\mathcal{S}_{\gamma}\subset\tilde{\mathcal{S}}_{\gamma}^{\kappa}$,
we would replace $\kappa$ by $\hat{\kappa}=\min\{g_{\hat{\theta}}(x):x\notin\mathcal{H}(T_{0})\}$ as in Step 1 of Algorithm \ref{algo:stage1}.

We give a theorem similar to Theorem \ref{thm: set_ERM} for this more sophisticated procedure. To this end, we begin by giving some definitions similar to the set up of
Theorem \ref{thm: set_ERM}. Let $R(f_{\theta}):=E_{X\sim q}\ell(f_{\theta}(X),I(X\in\mathcal{S}_{\gamma}))$ denote the
true risk function. Let $f^{*}=\arg\min_{f\in\mathcal{F}}R(f)$ denote
the true risk minimizer within function class $\mathcal{F}$. Define
$g^{*}=f_{1}^{*}-f_{0}^{*}$ accordingly and let $\kappa^{*}:=\min_{x\notin \mathcal{S}_{\gamma}}g^{*}(x)$
denote the true threshold associated with $f^*$ in obtaining the smallest outer rare-event approximation.

\begin{theorem} \label{thm: set_ERM_modified}  Suppose that the density $q$ has bounded support $K\subset[0,M]^d$ and $0<q_l\leq q(x)\leq q_u$ for any $x\in K$. Also suppose that there exists a function $h$
such that for any $f_{\theta}\in\mathcal{F}$, if $g_{\theta}(x)\geq\kappa$,
we have $\ell(f_{\theta}(x),0)\geq h(\kappa)>0$ (for the cross entropy loss, this happens if we know that $f_{\theta}$ has a bounded range). Then, for the set $\tilde{\mathcal{S}}_{\gamma}^{\hat{\kappa}}$, with probability at
least $1-\delta$,

\begin{align*}
 & P_{X\sim q}\left(X\in\tilde{\mathcal{S}}_{\gamma}^{\hat{\kappa}},X\in\mathcal{S}_{\gamma}^{c}\right)\\
\leq & \left(h(\kappa^{*}-t(\delta,n_1)\sqrt{d}\text{Lip}(g^{*})-\left\Vert \hat{g}-g^{*}\right\Vert _{\infty})\right)^{-1}\left(R(f^{*})+2\sup_{f_{\theta}\in\mathcal{F}}\left|R_{n_1}(f_{\theta})-R(f_{\theta})\right|\right)
\end{align*}

Here $\text{Lip}(g^*)$ is the Lipschitz parameter of $g^*$, and $t(\delta,n_1)$ is defined as in Theorem \ref{thm: set_ERM}.
\end{theorem}

\subsection{Implications of Theorem \ref{thm: set_ERM} and Related Results in the Literature}\label{subsec: explicit_bounds}

First, we explain the trade-off between overfitting and underfitting. If the function class $\mathcal{G}$ is not rich, then $R(g^{*})=\inf_{g\in\mathcal{G}}{R(g)}$ may be big because of the lack of expressive power. On the other hand, if the function class is too rich, then the generalization error will be huge.  Here, the generalization
error is represented by $\sup_{g_{\theta}\in\mathcal{G}}\left|R_{n_1}(g_{\theta})-R(g_{\theta})\right|$
as well as $t(\delta,n_1)\sqrt{d}\text{lip}(g^{*})+\left\Vert \hat{g}-g^{*}\right\Vert _{\infty}$,
which characterize the difference between the right hand side of the bound in the theorem and its limit as $n_1\rightarrow\infty$.

Another question is how to give a more refined bound for the false positive rate based on Theorem \ref{thm: set_ERM} that depends on explicit constants of the classification model or training process. This would involve theoretical results for deep neural networks that are under active research. Let us examine the terms appearing in
Theorem \ref{thm: set_ERM} and give some related results. In machine learning theory, the term $\sup_{g_{\theta}\in\mathcal{G}}\left|R_{n_1}(g_{\theta})-R(g_{\theta})\right|$
is often bounded by the Rademacher complexity of the function class
(some results about the Rademacher complexity for neural networks are in \citealt{pmlr-v65-harvey17a,NIPS2019_9246}).
The convergence of $\left\Vert \hat{g}-g^{*}\right\Vert _{\infty}$
to 0 as $n_1\rightarrow\infty$ is implied by the convergence of the
parameters, which is in turn justified by
the empirical process theory \citep{WeakConvergence_VdV}. A bound for
$\text{Lip}(g^{*})$ could be potentially derived by adding norm constraints
to the parameters in the neural network \citep{pmlr-v97-anil19a}. On the other hand, if we let the network size
grow to infinity, the class of neural networks can approximate any continuous
function \citep{NIPS2017_7203}, and hence $R(g^{*})$ can be arbitrarily
small when the neural network is complex enough.  
However, if we restrict the choices of networks, for instance by the Lipschitz constant, then no results regarding the sufficiency of its expressive power for arbitrary functions are available in the literature to our knowledge, and thus it appears open how to simultaneously give bounds for $\text{Lip}(g^{*})$ and  $R(g^{*})$. Future investigations on the expressive power of restricted classes of neural networks would help refining our conservativeness results further.

\section{Lower-Bound Efficiency Certificate and Estimators}
\label{app:lowerbound}

In Section \ref{sec:Deep-PrAE}, we described an approach that gives an estimator for the rare-event probability with an upper-bound relaxed efficiency certificate. Here we present analogous definitions and results on the lower-bound relaxed efficiency certificate. This lower-bound estimator gives an estimation gap for the upper-bound estimator. Moreover, by combining both of them, we can obtain an interval for the target rare-event probability. 

The lower-bound relaxed efficiency certificate is defined as follows (compare with Definition \ref{relaxed certificate}):

\begin{definition}
We say an estimator $\hat\mu_n$ satisfies an lower-bound \emph{relaxed efficiency certificate} to estimate $\mu$ if
$P(\hat\mu_n-\mu>\epsilon\mu)\leq\delta$
with $n\geq\tilde O(\log(1/\mu))$, for given $0<\epsilon,\delta<1$. \label{def: LB_relaxed_certificate}
\end{definition}

This definition requires that, with high probability, $\hat{\mu}_n$ is a lower bound of $\mu$ up to an error of $\epsilon\mu$. We have the following analog to Proposition \ref{certificate prop simple}:

\begin{corollary}

Suppose $\hat\mu_n$ is downward biased, i.e., $\overline\mu:=E[\hat\mu_n]\leq\mu$. Moreover, suppose $\hat\mu_n$ takes the form of an average of $n$ i.i.d. simulation runs $Z_i$, with $RE=Var(Z_i)/\overline\mu^2=\tilde O(\log(1/\overline\mu))$. Then $\hat\mu_n$ possesses the lower-bound relaxed efficiency certificate. \label{certificate prop simple_LB}
\end{corollary}

This motivates us to learn an inner approximation of the rare-event set in Stage 1 and then in Stage 2, we use IS as in Theorem \ref{general IS simplified} to estimate the probability of this inner approximation set. For the inner set approximation, like the outer approximation case, we use our Stage 1 samples $\{(\tilde X_i,  Y_i) \}_{i=1,\ldots,n_1}$ to construct an approximation set $\overline{\mathcal{S}}_{\gamma}$ that has zero false positive rate, i.e., 
\begin{equation}
P(X\in\overline{\mathcal{S}}_{\gamma},Y=0)=0.\label{eq: zero_FN}
\end{equation}
To make sure of (\ref{eq: zero_FN}), we again exploit the knowledge that the rare event set $\mathcal{S}_{\gamma}$ is orthogonally monotone.  Indeed, denote $T_1:=\{\tilde X_i : Y_i=1\}$ as the rare-event sampled points and for each point $x\in\mathbb{R}^d_+$, let $\mathcal Q(x):=\{x^{\prime} : x^{\prime}\geq x\}$. We construct $\mathcal J(T_1):=\cup_{x\in T_1}{\mathcal Q(x)}$ which serves as the ``upper orthogonal monotone hull" of $T_1$. The orthogonal monotonicity property of $\mathcal{S}_{\gamma}$ implies that $\mathcal J(T_1)\subset \mathcal{S}_{\gamma}$. Moreover, $\mathcal J(T_1)$ is the largest choice of $\overline{\mathcal{S}}_{\gamma}$ such that \eqref{eq: zero_FN} is guaranteed. Based on this observation, in parallel to Section \ref{sec:Deep-PrAE}, depending on how we construct the inner approximation to the rare-event set, we propose the following two approaches.

\textbf{Lazy-Learner IS (Lower Bound). }We now consider an estimator for $\mu$ where in Stage 1, we sample a constant $n_1$ i.i.d. random points from some density, say $q$. Then, we use the mixture IS depicted in Theorem 1 to estimate $P(X\in\mathcal J(T_1))$ in Stage 2. Since $\mathcal J(T_1)$ takes the form $\cup_{x\in T_1}{\mathcal Q(x)}$, it has a finite number of dominating points, which can be found by a sequential algorithm. But as explained in Section \ref{sec:Deep-PrAE}, this leads to a large number of mixture components that degrades the IS efficiency.

\textbf{Deep-Learning-Based IS (Lower Bound). }We train a neural network classifier, say $\hat g$, using all the Stage 1 samples $\{(\tilde X_i,Y_i)\}$, and obtain an approximate rare-event region $\overline{\mathcal S}_\gamma^\kappa=\{x:\hat g(x)\geq\kappa\}$, where $\kappa$ is say $1/2$. Then we adjust $\kappa$ minimally away from $1/2$, say to $\hat\kappa$, so that $\overline{\mathcal  S}_\gamma^{\hat\kappa}\subset\mathcal J(T_1)$, i.e., $\hat\kappa=\min\{\kappa\in\mathbb R: \overline{\mathcal  S}_\gamma^{\hat\kappa}\subset\mathcal J(T_1)\}$. Then $\overline{\mathcal  S}_\gamma^{\hat\kappa}$ is an inner approximation for $\mathcal S_\gamma$ (see Figure \ref{fig:illustration}(c), where $\hat\kappa=0.83$). Stage 1 in Algorithm \ref{algo: Deep-PrAE-LB} shows this procedure. With this, we can run mixture IS to estimate $P(X\in\overline{\mathcal  S}_\gamma^{\hat\kappa})$ in Stage 2.

\begin{algorithm}[h]
\KwIn{Black-box evaluator $I(\cdot\in\mathcal S_\gamma)$, initial Stage 1 samples $\{(\tilde X_i, Y_i) \}_{i=1,\ldots,n_1}$, Stage 2 sampling budget $n_2$, input distribution $N(\lambda,\Sigma)$.}
\KwOut{IS estimate $\hat\mu_n$.}

\nl \textbf{Stage 1 (Set Learning):}\\
\nl Train classifier with positive decision region $\overline{\mathcal  S}_\gamma^{\kappa}=\{x:\hat{g}(x) \geq\kappa\}$ using $\{(\tilde X_i, Y_i) \}_{i=1,\ldots,n_1}$;\\
\nl Replace $\kappa$ by $\hat\kappa=\min\{\kappa\in\mathbb R: \overline{\mathcal  S}_\gamma^{\hat\kappa}\subset \mathcal{J}(T_1)\}$;\\

\nl \textbf{Stage 2 (Mixture IS based on Searched dominating points):}\\
\nl The same as Stage 2 of Algorithm \ref{algo:stage1}.

    \caption{{\bf Deep-PrAE to estimate $\mu=P(X\in\mathcal S_\gamma)$ (lower bound).} 
    \label{algo: Deep-PrAE-LB}}
\end{algorithm}

As we can see, compared with Algorithm \ref{algo:stage1}, the only difference is how we adjust $\kappa$ in Stage 1. And similar to Theorem \ref{NN main}, we also have that Algorithm \ref{algo: Deep-PrAE-LB} attains the lower-bound relaxed efficiency certificate:

\begin{theorem}[Lower-bound relaxed efficiency certificate for deep-learning-based mixture IS]
Suppose $\mathcal S_\gamma$ is orthogonally monotone, and $\overline{\mathcal S}_\gamma^{\hat\kappa}$ satisfies the same conditions for $\mathcal S_{\gamma}$ in Theorem \ref{general IS simplified}. Then Algorithm \ref{algo: Deep-PrAE-LB} attains the lower-bound relaxed efficiency certificate by using a constant number of Stage 1 samples.\label{thm: LB_achieved_by_algo}
\end{theorem}

Finally, we investigate the conservativeness of this bound, which is measured by the false negative rate $P(X\notin\overline{\mathcal{S}}_{\gamma}^{\hat{k}},Y=1)$.  Like in Section \ref{sec:Deep-PrAE}, we use ERM to train $\hat g$, i.e., $\hat g:=\text{argmin}_{g\in\mathcal{G}}\{ {R}_{n_1}(g):=\frac{1}{n_1}\sum_{i=1}^{n_1}\ell(g(\tilde X_{i}),Y_{i})\}$
where $\ell$ is a loss function and $\mathcal G$ is the considered hypothesis class. Let $g^*$ be the true risk minimizer as described in Section \ref{sec:Deep-PrAE}. For inner approximation, we let $\kappa^{*}:=\max_{x\in\mathcal{S}_{\gamma}^c}g^{*}(x)$ be
the true threshold associated with $g^{*}$ in obtaining the largest inner rare-event set approximation. Then we have the following result analogous to Theorem \ref{thm: set_ERM}.

\begin{theorem}[Lower-bound estimation conservativeness]\label{thm: inn_set_ERM} 
Consider $\overline{\mathcal S}_\gamma^{\hat\kappa}$ obtained in Algorithm \ref{algo: Deep-PrAE-LB} where $\hat g$ is trained from an ERM. 
Suppose  the density $q$ has bounded 
support $K\subset[0,M]^{d}$ and $0<q_{l}\leq q(x)\leq q_{u}$ for any $x\in K$. Also suppose there exists
a function $h$ such that for any $g\in\mathcal{G}$, $g(x)\leq\kappa$ implies $\ell(g(x),1)\geq h(\kappa)>0$.
(e.g., if $\ell$ is the squared loss, then $h(\kappa)$ could be chosen 
as $h(\kappa)=(1-\kappa)^{2}$). Then, with probability at least $1-\delta$,
\begin{equation*}
P_{X\sim q}\left(X\in\overline{\mathcal{S}}_{\gamma}^{\hat{\kappa}}\setminus\mathcal{S}_{\gamma}\right)
\leq \frac{R(g^{*})+2\sup_{g\in\mathcal{G}}\left|R_{n_1}(g)-R(g)\right|}{h(\kappa^{*}+t(\delta,n_1)\sqrt{d} \text{Lip}(g^{*})+\left\Vert \hat{g}-g^{*}\right\Vert _{\infty})}.
\end{equation*}
Here, $\text{Lip}(g^{*})$ is the Lipschitz parameter of 
$g^{*}$, and $t(\delta,n_1)=3\left(\frac{\log(n_1q_{l})+d\log M+\log\frac{1}{\delta}}{n_1q_{l}}\right)^{\frac{1}{d}}$.
\end{theorem}

\section{Cross Entropy and Adaptive Multilevel Splitting}\label{app:ce_ams}

We provide some details on the cross-entropy method and adaptive multilevel splitting (or subset simulation), and also discuss their challenges in black-box problems.

\paragraph{Cross Entropy.}
The cross-entropy (CE) method \citep{de2005tutorial,rubinstein2013cross} uses a sequential optimization approach to iteratively solve for the optimal parameter in a parametric class of IS distributions. The objective in this optimization sequence is to minimize the Kullback–Leibler divergence between the IS distribution and the zero-variance IS distribution (the latter is theoretically known to be the conditional distribution given the occurrence of the rare event, but is unimplementable as it requires knowing the rare-event probability itself). Specifically, assume we are interested in estimating $P(g(X)>\gamma)$ and a parametric class $p_\theta$ is considered. The cross-entropy method adaptively chooses $\gamma_1<\gamma_2<...<\gamma$. At each intermediate level $k$, we use the updated IS distribution $p_{\theta^*_k}$, designed for simulating $P(g(X)>\gamma_k)$, as the sampling distribution to draw samples of $X$ that sets up an empirical optimization, from which the next $\theta^*_{k+1}$ is obtained.

While flexible and easy to use, the efficiency of CE depends crucially on the expressiveness of the parametric class $p_\theta$ and the parameter convergence induced by the empirical optimization sequence. There are good approaches to determine the parametric classes (e.g., \citealt{botev2016semiparametric}), and also studies on the efficiency of IS distributions parametrized by empirical optimization \citep{tuffin2012probabilistic}. However, it is challenging to obtain an efficiency certificate for CE that requires iterative empirical optimization in the common form depicted above. Insufficiency on either the choice of the parametric class or the parameter convergence may lead to the undetectable under-estimation issue (e.g., as in Theorem~\ref{counterexample}). 

\paragraph{Adaptive Multilevel Splitting.}
Adaptive multilevel splitting (AMS) (or subset simulation) \citep{cerou2007adaptive,au2001estimation} decomposes the rare-event estimation problem into estimating a sequence of conditional probabilities. We adaptively choose a threshold sequence $\gamma_1<\gamma_2<...<\gamma_{K}=\gamma$. Then $P(g(x)>\gamma)$ can be rewritten as  $P(g(x)>\gamma)=P(g(x)>\gamma_1) \prod_{k=2}^{K} P(g(x)>\gamma_{k}|g(x)>\gamma_{k-1})$. AMS then aims to estimate $P(g(x)>\gamma_1)$ and $P(g(x)>\gamma_{k}|g(x)>\gamma_{k-1})$ for each intermediate level $k=2,...,K$. In standard implementation, these conditional probabilities are estimated using samples from $p(g(x)>\gamma_{k}|g(x)>\gamma_{k-1})$ through variants of the Metropolis-Hasting (MH) algorithms.

Theoretical studies have shown some nice properties of AMS, including unbiasedness and asymptotic normality (e.g., see \citealt{cerou2019adaptive}). However, the variance of the estimator depends on the mixing property of the proposal distribution in the MH steps \citep{cerou2016fluctuation}. Under ideal settings when direct sampling from $P(g(x)>\gamma_{k}|g(x)>\gamma_{k-1})$ is possible, it is shown that AMS is ``almost'' asymptotically optimal \citep{guyader2011simulation}. However, to our best knowledge, there is yet any study on provable efficiency of rare-event estimators with consideration of both AMS and MH sampling. In practice, to achieve a good performance, AMS requires a proposal distribution in the MH algorithm that can efficiently generate samples with low correlations.

\section{Further Details for Numerical Experiments}
\label{app:experiment}

This section provides more details on the two experimental examples in Section \ref{sec:numerics}.

\subsection{2D Example}
In the 2D example, the rarity parameter $\gamma$ governs the shape of the rare-event set $\mathcal S_\gamma=\{x:g(x)\geq\gamma\}$. We consider a linear combination of sigmoid functions
$g(x) = \|\theta_1\psi(x-c_1-\gamma) +\theta_2\psi(x-c_2-\gamma)+\theta_3\psi(x-c_3-\gamma)+\theta_4\psi(x-c_4-\gamma)\|$ where $\theta,c$ are some constant vectors and $\psi(x) = \frac{ \exp(x)}{1+\exp(x)}$. A point $x$ is a rare-event if $g(x) > \gamma$, where we take $\gamma=1.8$ in Section \ref{sec:numerics}.  We use $p = N([5, 5]^T, 0.25I_{2 \times 2})$. Figure \ref{fig:exp_2d_rare_set} shows the rare-event set and its approximations for various $\gamma$'s. The Deep-PrAE boundaries seem tight in most cases, attributed to both the sufficiently trained NN classifier and the bisection algorithm implemented for tuning $\hat \kappa$ after the NN training.

\begin{figure}[h]
    \begin{subfigure}{\textwidth}
      \begin{subfigure}{.33\textwidth}
      \centering
        \includegraphics[width=\textwidth]{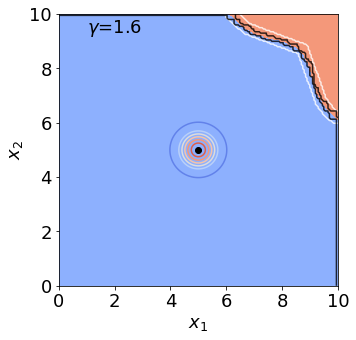}
        \caption{}
      \end{subfigure}%
       \begin{subfigure}{.33\textwidth}
      \centering
        \includegraphics[width=\textwidth]{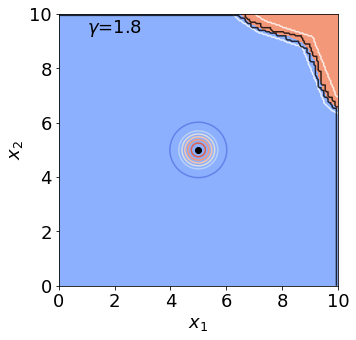}
        \caption{}
     \end{subfigure}
      \begin{subfigure}{.33\textwidth}
      \centering
        \includegraphics[width=\textwidth]{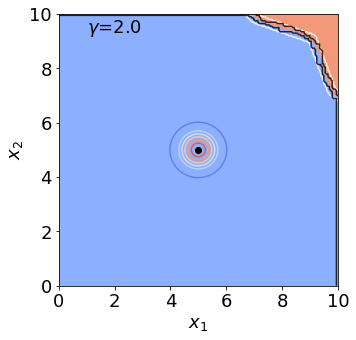}
        \caption{}
     \end{subfigure}
    \end{subfigure}
     
    \begin{subfigure}{\textwidth}
      \begin{subfigure}{.33\textwidth}
      \centering
        \includegraphics[width=\textwidth]{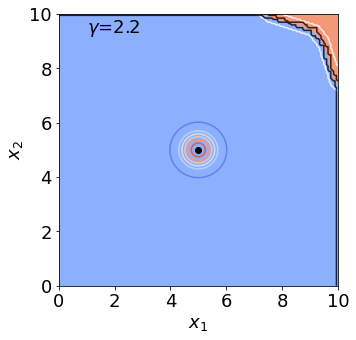}
        \caption{}
      \end{subfigure}%
       \begin{subfigure}{.33\textwidth}
      \centering
        \includegraphics[width=\textwidth]{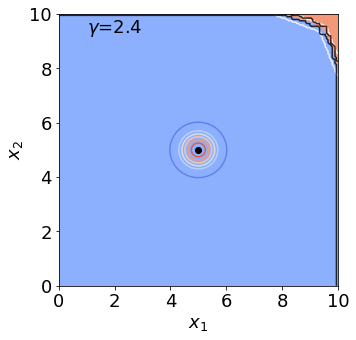}
        \caption{}
     \end{subfigure}
      \begin{subfigure}{.33\textwidth}
      \centering
        \includegraphics[width=\textwidth]{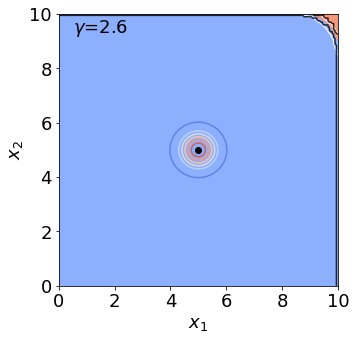}
        \caption{}
     \end{subfigure}
     
   \end{subfigure}
  
  \caption{The contour of $p$, rare-event set $\mathcal S_\gamma$ (reddish region), outer- and inner- approximation boundaries (black lines) and Deep-PrAE UB and LB decision boundaries (white lines) for some $\gamma$ values in the 2D example. }
  \label{fig:exp_2d_rare_set}
\end{figure}

\subsection{Intelligent Driving Safety Testing Example}
We provide more details about the self-driving example, which simulates the interaction of an autonomous vehicle (AV) model that follows a human-driven lead vehicle (LV). The AV is controlled by the Intelligent Driver Model (IDM), widely used for autonomy evaluation and microscopic transportation simulation, that maintains a safety distance while ensuring smooth ride and maximum efficiency. The states of the AV are $s_t = [x_{\text{follow}}, x_{\text{lead}}, v_{\text{follow}}, v_{\text{lead}}, a_{\text{follow}}, a_{\text{lead}}]_t$ which are the position, velocity and acceleration of the AV and LV respectively. The throttle input to the AV is defined as $u_t$ which has an affine relationship with the acceleration of the vehicle. Similarly, the randomized throttle of the LV is represented by $w_t$. With a car length of $L$, the distance between the LV and AV at time $t$ is given by $r_t=x_{\text{lead}, t}-x_{\text{follow}, t}-L$, which has to remain below the crash threshold for safety.We describe the dynamics in more detail below. Figure \ref{fig:sim_state} gives a pictorial overview of the interaction.

\begin{figure}[h]
  \centering
    \includegraphics[width=0.8\textwidth]{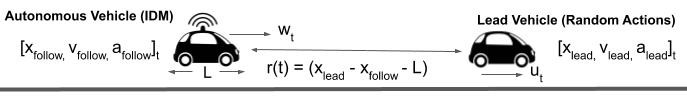}
    \caption{The states $s_t$ and input $u_t$ of the self-driving safety-testing simulation. $w_t$ denotes the throttle input of the AV from the IDM.
    }
    \label{fig:sim_state}
\end{figure} 

\paragraph{LV actions.}
The LV action contains human-driving uncertainty in decision-making modeled as Gaussian increments. For every $\Delta t$ time-steps, a Gaussian random variable is generated with the mean centered at the previous action $u_{t-\Delta t}$. We initialize $u_0 = 10$ (unitless) and $\Delta t = 4$ sec, which corresponds to zero initial acceleration and an acceleration change in the LV once every 4 seconds.

\paragraph{Intelligent Driver Model (IDM) for AV.} The IDM is governed by the following equations (the subscripts ``follow'' and ``lead'' defined in Figure \ref{fig:sim_state} is abbreviated to ``f'' and ``l'' for conciseness):
\begin{align*}
    \dot x_{f} &= v_{f}\\
    \dot x_{l} &= v_{l}\\
     \dot v_{f} & = \max \left(a(1-\left(\frac{v_{f}}{v_0}\right)^{\delta}-\left(\frac{s^{*}(v_{f},\Delta v_{f})}{s_{f}}\right)^{2}), -d\right)\\
    s^{*}(v_{f},\Delta v_{f}) &= s_0 + v_{f}\bar T + \frac{v_{f}\Delta v_{f}}{2\sqrt{ab}}\\
    s_{f} &= x_{l} - x_{f} - L\\
    \Delta v_{f} &= v_{f} - v_{l},
\end{align*}
The parameters are presented in Table \ref{idm-params}, and $v_l \propto u_t$ and $v_f \propto w_t$, . The randomness of LV actions $u_t$'s propagates into the system and affects all the simulation states $s_t$.
The IDM is governed by simple first-order kinematic equations for the position and velocity of the vehicles. The acceleration of the AV is the decision variable where it is defined by a sum of non-linear terms which dictate the ``free-road'' and ``interaction'' behaviors of the AV and LV. The acceleration of the AV is constructed in such a way that certain terms of the equations dominate when the LV is far away from the AV to influence its actions and other terms dominate when the LV is in close proximity to the AV.

\begin{table}
\caption{Parameters of the Intelligent Drivers Model (IDM)}
\label{idm-params}
\centering
\begin{tabular}{ll}
\toprule
Parameters                         & Value              \\
\midrule
Safety distance, $s_0$                                    & 2 m                         \\ 
Speed of AV in free traffic, $v_0$                        & 30 m/s                      \\ 
Maximum acceleration of AV, $a$                          & $2\gamma$ m/s$^2$    \\ 
Comfortable deceleration of AV, $b$                      & 1.67 m/s$^2$ \\ 
Maximum deceleration of AV, $d$                          & $2\gamma$ m/s$^2$    \\ 
Safe time headway, $\bar T$                                   & 1.5 s                       \\ 
Acceleration exponent parameter, $\delta$ & 4                           \\
Car length, $L$                                          & 4 m                         \\
\bottomrule
\end{tabular}
\end{table}

\paragraph{Rarity parameter $\gamma$.} Parameter $\gamma$ signifies the range invoked by the AV acceleration and deceleration pedals. Increasing $\gamma$ implies that the AV can have sudden high deceleration and hence avoid crash scenarios better and making crashes rarer. In contrast, decreasing $\gamma$ reduces the braking capability of the AV and more easily leads to crashes. For instance, $\gamma=1.0$ corresponds to AV actions in the range $[5, 15]$ or correspondingly $a_{\text{follow}, t} \in [-2, 2]$, and  $\gamma=2.0$ corresponds to $a_{\text{follow}, t} \in [-4, 4]$. Figure \ref{fig:lv_slices_gamma} shows the approximate rare-event set by randomly sampling points and evaluating the inclusion in the set, for the two cases of $\gamma=1.0$ and $\gamma=2.0$. In particular, we slice the 15-dimensional space onto pairs from five of the dimensions. In all plots, we see that the crash set (red) are monotone, thus supporting the use of our Deep-PrAE framework. Although the crash set is not located in the ``upper-right corner'', we can implement Deep-PrAE framework for such problems by simple re-orientation.

\begin{figure}[h]
    \begin{subfigure}{\textwidth}
       \begin{subfigure}{.49\textwidth}
      \centering
        \includegraphics[width=\textwidth]{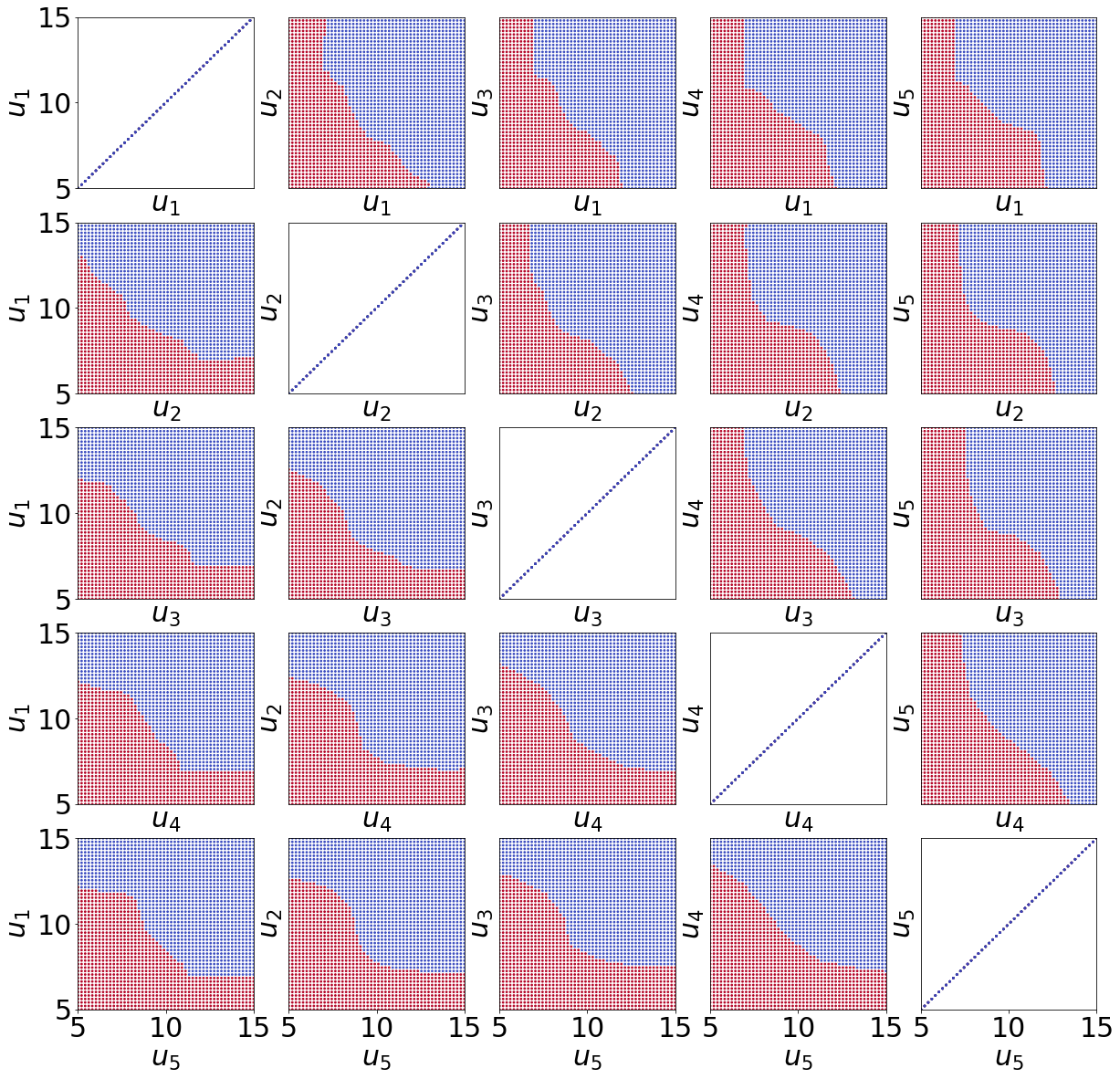}
        \caption{Case $\gamma=1.0$}
        \label{fig:lv_slices_gamma10}
     \end{subfigure}
      \begin{subfigure}{.49\textwidth}
      \centering
        \includegraphics[width=\textwidth]{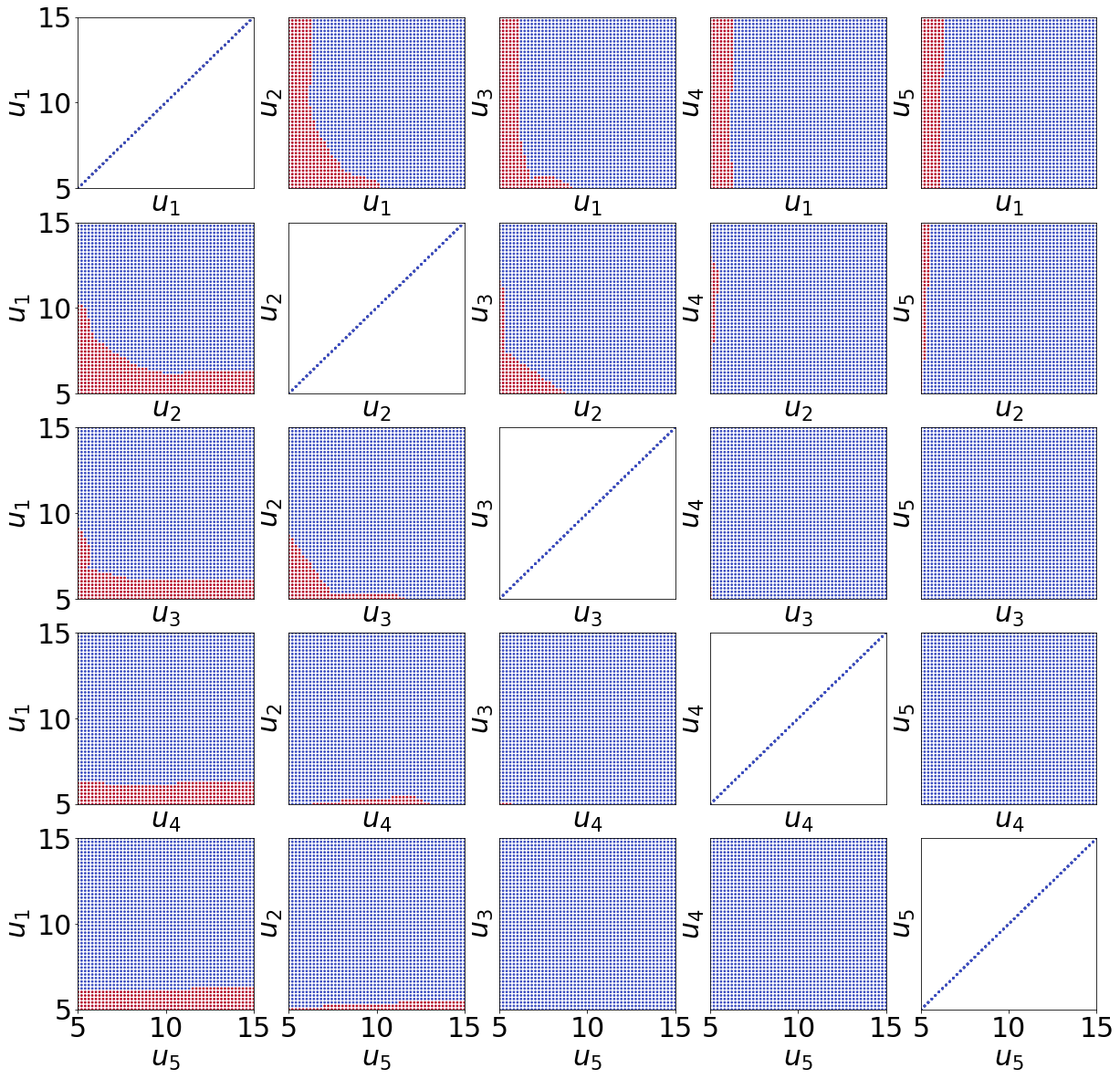}
        \caption{Case $\gamma=2.0$}
        \label{fig:lv_slices_gamma20}
      \end{subfigure}%
   \end{subfigure}
  \centering
    \caption{Slice of pairs of the first 5 dimensions of LV action space. For any $(u_i, u_{i^\prime})$ shown, $u_j, j \not \in \{i, i^\prime\}$ is fixed at a constant value. Blue dots = non-crash cases, red dots= crash cases.}
  
  \label{fig:lv_slices_gamma}
\end{figure}

\textbf{Sample trajectories. }Figure \ref{fig:carfollowtraj} shows two examples of sample trajectories, one successfully maintaining a safe distance, and the other leading to a crash. In Figure \ref{fig:carfollowtraj}(e)-(h) where we show the crash case, the AV maintains a safe distance behind the LV until the latter starts rapidly decelerating (Figure \ref{fig:carfollowtraj}(h)). Here the action corresponds to the throttle input that has an affine relationship with the acceleration of the vehicle. The LV ultimately decelerates at a rate that the AV cannot attain and its deceleration saturates after a point which leads to the crash.

\begin{figure}[h]
  \centering
      \begin{subfigure}{.22\textwidth}
      \centering
        \includegraphics[width=\textwidth]{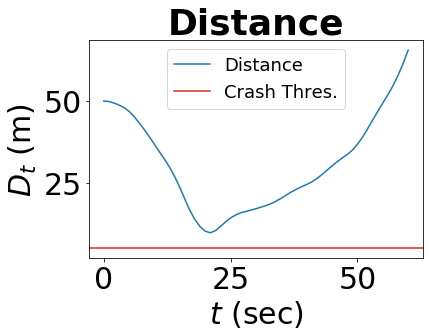}
        \caption{}
      \end{subfigure}
      \begin{subfigure}{.22\textwidth}
      \centering
        \includegraphics[width=\textwidth]{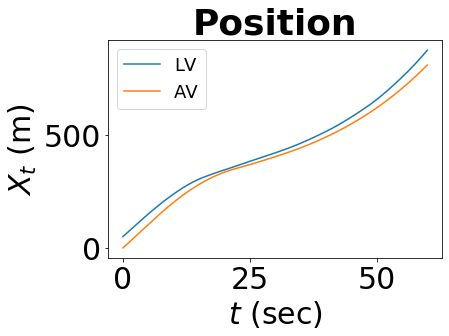}
        \caption{}
      \end{subfigure}
      \begin{subfigure}{.22\textwidth}
      \centering
        \includegraphics[width=\textwidth]{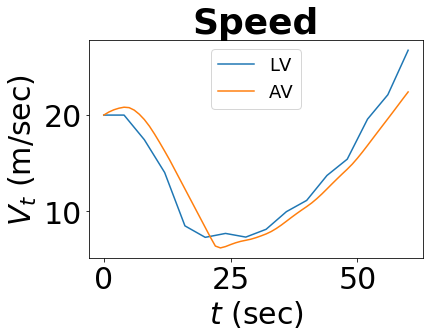}
        \caption{}
      \end{subfigure}
       \begin{subfigure}{.22\textwidth}
      \centering
        \includegraphics[width=\textwidth]{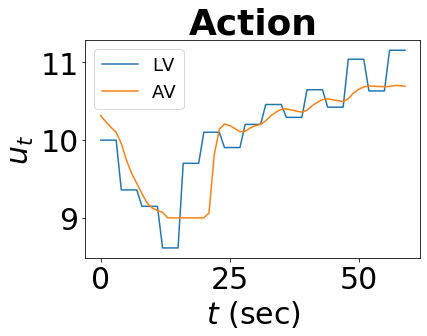}
        \caption{}
      \end{subfigure}
        \begin{subfigure}{.22\textwidth}
      \centering
        \includegraphics[width=\textwidth]{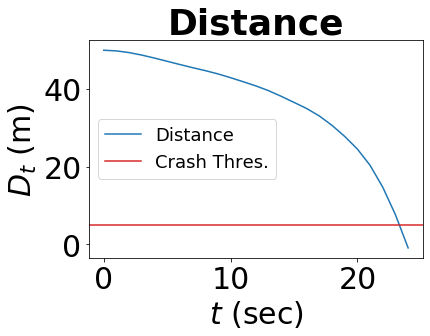}
        \caption{}
      \end{subfigure}
      \begin{subfigure}{.22\textwidth}
      \centering
        \includegraphics[width=\textwidth]{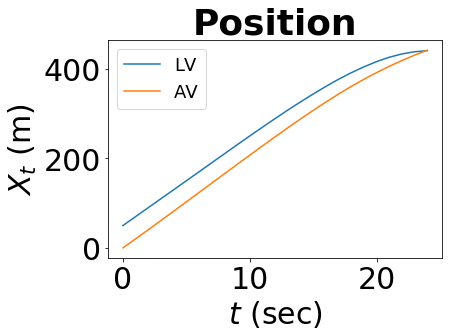}
        \caption{}
      \end{subfigure}
      \begin{subfigure}{.22\textwidth}
      \centering
        \includegraphics[width=\textwidth]{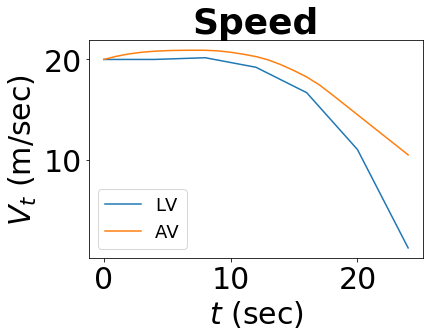}
        \caption{}
      \end{subfigure}
       \begin{subfigure}{.22\textwidth}
      \centering
        \includegraphics[width=\textwidth]{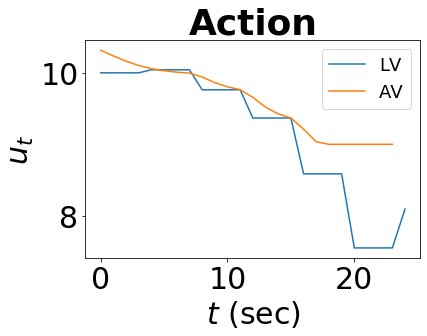}
        \caption{}
      \end{subfigure}
    \caption{Autonomous Car Following Experiment Trajectories. Figures (a) - (d) represent a simulation episode without a crash occurring where the AV follows the LV successfully at a safe distance. Figures (e) - (h) represents a simulation episode where crash occurs at $t=23$ seconds due to the repeated deceleration of the LV.}
  \label{fig:carfollowtraj}
\end{figure}

\subsection{Code}
The code and environment settings for the experiments are available at
\href{https://github.com/safeai-lab/Deep-PrAE/}{https://github.com/safeai-lab/Deep-PrAE/}.

\section{Proofs}\label{app:proofs}

\subsection{Proofs for the Dominating Point Methodologies}

\begin{proof}[Proof of Proposition \ref{naive Monte Carlo}]

Since $\mu$ is exponentially decaying in $\gamma$ while $n$ is polynomially growing in $\gamma$, we know that $\lim_{\gamma\rightarrow\infty}n\mu=0$. Since $n\hat{\mu}_n$ takes values in $\{0,1,\dots,n\}$, we get that $P(|\hat{\mu}_n-\mu|>\varepsilon\mu)=P(|n\hat{\mu}_n-n\mu|>\varepsilon n\mu)\rightarrow 1$ as $\gamma\rightarrow\infty$.
\end{proof}
\begin{proof}[Proof of Theorem \ref{general IS}]
Throughout this proof, we write $f(\gamma)\sim g(\gamma)$ if $f(\gamma)/g(\gamma)$ changes at most polynomially in $\gamma$. We know that 
$$
\tilde{E}[Z^2]=\sum_j \tilde{E}[I(X\in \mathcal S_\gamma^j)L^2(X)]\leq \sum_j e^{-(a_j-\lambda)^T\Sigma^{-1}(a_j-\lambda)}/\alpha_j\sim e^{-(a^*-\lambda)^T\Sigma^{-1}(a^*-\lambda)}.
$$
Denote $Y=B(X-\lambda)\sim N(0, B\Sigma B^T)$ and $s=B(a^*-\lambda)$. Define $\tilde{\varepsilon}=\varepsilon\min_{u:u^T(B\Sigma B^T)^{-1}u=1}\|u\|_{\infty}$. Then we also know that 
$$
\begin{aligned}
&\tilde{E}[I(X\in \mathcal S_\gamma)L(X)]\\
\geq & P(B(X-a^*)\geq 0,(X-a^*)^T\Sigma^{-1}(X-a^*)\leq \varepsilon^2)\\
= & P(Y\geq s,(Y-s)^T(B\Sigma B^T)^{-1}(Y-s)\leq \varepsilon^2)\\
= & \int_{y\geq s,(y-s)^T(B\Sigma B^T)^{-1}(y-s)\leq\varepsilon^2} (2\pi)^{-d/2}|B\Sigma B^T|^{-1/2}e^{-y^T(B\Sigma B^T)^{-1}y/2}\mathrm{d}y\\
\geq & (2\pi)^{-d/2}|B\Sigma B^T|^{-1/2}e^{-\varepsilon^2/2}e^{-(a^*-\lambda)^T\Sigma^{-1}(a^*-\lambda)/2}\\
&\int_{y\geq s,(y-s)^T(B\Sigma B^T)^{-1}(y-s)\leq\varepsilon^2}e^{-s^T(B\Sigma B^T)^{-1}(y-s)}\mathrm{d}y\\
\geq & (2\pi)^{-d/2}|B\Sigma B^T|^{-1/2}e^{-\varepsilon^2/2}e^{-(a^*-\lambda)^T\Sigma^{-1}(a^*-\lambda)/2}\prod_{i=1}^d \int_{0}^{\tilde{\varepsilon}} e^{-s^T(B\Sigma B^T)^{-1}e_iu_i}\mathrm{d}u_i\\
=&(2\pi)^{-d/2}|B\Sigma B^T|^{-1/2}e^{-\varepsilon^2/2}e^{-(a^*-\lambda)^T\Sigma^{-1}(a^*-\lambda)/2}\prod_{i=1}^d \frac{1-e^{-s^T(B\Sigma B^T)^{-1}e_i\tilde{\varepsilon}}}{s^T(B\Sigma B^T)^{-1}e_i}.
\end{aligned}
$$ 
Note that it is easy to verify that $s^T(B\Sigma B^T)^{-1}e_i\geq 0$. If $s^T(B\Sigma B^T)^{-1}e_i=0$, then we naturally use $\tilde{\varepsilon}$ to substitute $\frac{1-e^{-s^T(B\Sigma B^T)^{-1}e_i\tilde{\varepsilon}}}{s^T(B\Sigma B^T)^{-1}e_i}$. Since we have assumed that the components of $a^*$ are at most polynomially growing in $\gamma$, finally we get that 
$$
\tilde{E}[I(X\in \mathcal S_\gamma)L(X)]\sim  e^{-(a^*-\lambda)^T\Sigma^{-1}(a^*-\lambda)/2}
$$
and hence $\tilde{E}[Z^2]/\tilde{E}[Z]^2$ is at most polynomially growing in $\gamma$.

\end{proof}

\begin{proof}[Proof of Theorem \ref{counterexample}]
We know that $\tilde{E}[Z]=\bar{\Phi}(\gamma)+\bar{\Phi}(k\gamma)$. Moreover, 
    $$
    \tilde{E}[Z^2]=e^{\gamma^2}(\bar{\Phi}(2\gamma)+\bar{\Phi}((k-1)\gamma)).
    $$
    If $0<k\leq 1$, then $\tilde{E}[Z]=O\left( e^{-k^2\gamma^2/2}/\gamma\right)$ and $\tilde{E}[Z^2]=O\left( e^{\gamma^2}\right)$ as $\gamma\rightarrow\infty$. If $1<k<3$, then $\tilde{E}[Z]=O\left( e^{-\gamma^2/2}/\gamma\right)$ and $\tilde{E}[Z^2]=O\left( e^{(1-(k-1)^2/2)\gamma^2}/\gamma\right)$ as $\gamma\rightarrow\infty$. In both cases, we get that $\tilde{E}[Z^2]/\tilde{E}[Z]^2$ grows exponentially in $\gamma$. On the other hand, we know that 
    $$
    \begin{aligned}
    &\tilde{P}\left(\left|\frac{1}{n}\sum_i Z_i-\bar{\Phi}(\gamma)\right|>\varepsilon\bar{\Phi}(\gamma)\right)\\
    \leq& \tilde{P}(\exists i:X_i\leq -k\gamma)+\tilde{P}\left(\left|\frac{1}{n}\sum_i I(X_i\geq\gamma)e^{\gamma^2/2-\gamma X_i}-\bar{\Phi}(\gamma)\right|>\varepsilon\bar{\Phi}(\gamma)\right).
    \end{aligned}
    $$
    Clearly $\tilde{P}(\exists i:X_i\leq -k\gamma)=1-(1-\bar{\Phi}((k+1)\gamma))^n=O\left( n\bar{\Phi}((k+1)\gamma)\right)$, which is exponentially decreasing in $\gamma$ as $n$ is polynomial in $\gamma$. Moreover, by Chebyshev's inequality,
    $$
    \begin{aligned}
    &\tilde{P}\left(\left|\frac{1}{n}\sum_i I(X_i\geq\gamma)e^{\gamma^2/2-\gamma X_i}-\bar{\Phi}(\gamma)\right|>\varepsilon\bar{\Phi}(\gamma)\right)\\
    \leq & \frac{\tilde{E}[I(X_i\geq\gamma)e^{\gamma^2-2\gamma X_i}]}{n\varepsilon^2\bar{\Phi}^2(\gamma)}=\frac{e^{\gamma^2}\bar{\Phi}(2\gamma)}{n\varepsilon^2\bar{\Phi}^2(\gamma)}=O\left( \frac{\gamma}{n\varepsilon^2}\right).
    \end{aligned}
    $$
    Thus $P(|\hat{\mu}_n-\bar{\Phi}(\gamma)|>\varepsilon\bar{\Phi}(\gamma))=O\left(\frac{\gamma}{n\varepsilon^2}\right)$. Moreover, we know that $P(\exists i:Z_i>0)\geq 1-1/2^n$ and if $Z_i>0$ for some $i$, then we have that 
    $$
    \frac{\sum_{i}Z_i^2/n}{(\sum_{i}Z_i/n)^2}\leq n^2.
    $$
\end{proof}

\subsection{Proofs for the Relaxed Efficiency Certificate}

\begin{proof}[Proof of Proposition \ref{certificate prop simple}]
We have
$$P(\hat\mu_n-\mu<-\epsilon\mu)\leq P(\hat\mu_n-\overline\mu<-\epsilon\overline\mu)$$
since $\overline\mu\geq\mu$ and $1-\epsilon>0$. Note that the Markov inequality gives
$$P(\hat\mu_n-\overline\mu<-\epsilon\overline\mu)\leq\frac{\widetilde{Var}(Z_i)}{n\epsilon^2\overline\mu^2}$$
so that 
$$n\geq\frac{\widetilde{Var}(Z_i)}{\delta\epsilon^2\overline\mu^2}=\frac{RE}{\delta\epsilon^2}=\tilde O\left(\log\frac{1}{\overline\mu}\right)=\tilde O\left(\log\frac{1}{\mu}\right)$$
achieves the relaxed efficiency certificate.
\end{proof}

\begin{proof}[Proof of Proposition \ref{prop:extend}]
The proof follows from that of Proposition \ref{certificate prop simple} with a conditioning on $D_{n_1}$. We have
\begin{align*}
P(\hat\mu_n-\mu<-\epsilon\mu|D_{n_1})&\leq P(\hat\mu_n-\overline\mu(D_{n_1})<-\epsilon\overline\mu(D_{n_1})|D_{n_1})
\end{align*}
since $\overline\mu(D_{n_1})\geq\mu$ almost surely and $1-\epsilon>0$. Note that the Markov inequality gives
$$P(\hat\mu_n-\overline\mu(D_{n_1})<-\epsilon\overline\mu(D_{n_1})|D_{n_1})\leq\frac{Var(Z_i|D_{n_1})}{n_2\epsilon^2\overline\mu(D_{n_1})^2}$$
so that 
$$n_2\geq\frac{Var(Z_i|D_{n_1})}{\delta\epsilon^2\overline\mu(D_{n_1})^2}=\frac{RE(D_{n_1})}{\delta\epsilon^2}=\tilde O\left(\log\left(\frac{1}{\overline\mu(D_{n_1})}\right)\right)=\tilde O\left(\log\frac{1}{\mu}\right)$$
almost surely. Thus,
$$n=n_1+n_2\geq\tilde O\left(\log\frac{1}{\mu}\right)$$
achieves the relaxed efficiency certificate.

\end{proof}

\begin{proof}[Proof of Corollary \ref{relaxed prob}]
Follows directly from Proposition \ref{prop:extend}, since $\overline{\mathcal S}_\gamma\supset\mathcal S_\gamma$ implies $\overline\mu(D_{n_1})\geq\mu$ almost surely. 
\end{proof}

\begin{proof}[Proof of Theorem \ref{NN main}]
 We have assumed that $\overline{\mathcal S}_\gamma^{\hat\kappa}$ satisfies the assumptions for $\mathcal S_\gamma$ in Theorem \ref{general IS}. Then following the proof of Theorem \ref{general IS}, we obtain the efficiency certificate for the IS estimator in estimating its mean. Theorem \ref{NN main} is then proved by directly applying Corollary \ref{relaxed prob}.
\end{proof}

\subsection{Proofs for Conservativeness}

Recall that $T_0=\{\tilde{X}_i:Y_i=0\}$ where the samples are generated as in Algorithm \ref{algo:stage1}. By some combinitorial argument, we can prove the following lemma which
says that with high probability, each point in  $\mathcal{S}_{\gamma}^c$ that has sufficient distance
to its boundary could be covered by $\mathcal{H}(T_{0})$. 
\begin{lemma} Suppose that the density $q$ has bounded support $K\subset[0,M]^{d}$,
and for any $x\in K$, suppose that $0<q_{l}\leq q(x)\leq q_{u}$.
Define $B_{t}:=\{x\in\mathcal{S}_{\gamma}^{c}:x+t\mathbf{1}_{d\times1}\in\mathcal{S}_{\gamma}^{c}\}$.
Then with probability at least $1-\delta$, we have that $B_{t(\delta,n_1)}\subset\mathcal{H}(T_{0})$.
Here $t(\delta,n_1)=3\left(\frac{\log(n_1q_{l})+d\log{M}+\log\frac{1}{\delta}}{n_1q_{l}}\right)^{\frac{1}{d}}.$\label{lem: distance}
\end{lemma}

\begin{proof} The basic idea is to construct a finite number of regions,
such that when there is at least one sample point in each of these
regions, we would have that $B_{t}\subset\mathcal{H}(T_{0})$. Then
we could give a lower bound to the probability of $B_{t}\subset\mathcal{H}(T_{0})$
in terms of the number of regions and the volume of each of these
regions.

By dividing the first $d-1$ coordinates into $\frac{M}{\delta}$
equal parts, we partition the region $[0,M]^{d}$ into rectangles,
each with side length $\delta$, except for the $d-$th dimension
(the $\delta$ here is not exactly the $\delta$ in the statement
of the lemma, since we will do a change of variable in the last step).
To be more precise, the rectangles are given by 
\[
Z_{j}=\left(\prod_{i=1}^{d-1}[(j_{i}-1)\delta,j_{i}\delta]\right)\times[0,M].
\]
Here $j\in J$ and $J$ is defined by 
\[
J:=\{j=(j_{1},\cdots,j_{d-1}),j_{i}=1,2,\cdots,\frac{M}{\delta}\}.
\]
Denote by $J_{0}$ the set which consists of $j\in J$ such that there
exist a point in $B_{2\delta}$ whose first $d-1$ coordinates are
$j_{1}\delta,j_{2}\delta,\cdots,j_{d-1}\delta$ respectively, i.e.,
$J_{0}=\left\{ j\in J:B_{2\delta}\cap\left(\left(\prod_{i=1}^{d-1}\{j_{i}\delta\}\right)\times[0,M]\right)\neq\emptyset\right\} .$
For all $j\in J_{0}$, let $p_{j}$ be the point such that

i) $p_{j}\in B_{\delta}$

ii) The first $d-1$ coordinates of $p_{j}$ are $j_{1}\delta,j_{2}\delta,\cdots,j_{d-1}\delta$
respectively

iii) $p_{j}$ has $d-$th coordinate larger than $-\delta+\sup_{p\text{ satisfies i),ii)}}\left(d\text{-th coordinate of \ensuremath{p}}\right)$.

From the definition of $J_{0}$ and the fact that $B_{\delta}\supset B_{2\delta}$,
$p_{j}$ is guaranteed to exist. We claim that $B_{2\delta}\cap Z_{j}\subset\mathcal{R}(p_{j})$,
where $\mathcal{R}({p_j})$ is the rectangle that contains 0 and $p_j$ as two of its corners. Clearly, from the definition of $Z_{j}$, for
any point $x\in B_{2\delta}\cap Z_{j}$, its first $d-1$ coordinates
are smaller than $j_{1}\delta,j_{2}\delta,\cdots,j_{d-1}\delta$ respectively.
For the $d-$th coordinate, suppose on the contrary that there exists
$x\in B_{2\delta}\cap Z_{j}$ with $d-$th coordinate greater than
the $d-$th coordinate of $p_{j}$. Since $x\in Z_{j}$, the first
$d-1$ coordinates of $x$ are at least $(j_{1}-1)\delta,(j_{2}-1)\delta,\cdots,(j_{d-1}-1)\delta$,
so we have that $x+\delta\mathbf{1}_{d\times1}\geq p_{j}+\delta e_{d}$.
Since $x\in B_{2\delta}$, we know that $x+2\delta\mathbf{1}_{d\times1}\in\mathcal{S}_{\gamma}^{c}$.
Hence by the previous inequality and the orthogonal monotonicity of
$\mathcal{S}_{\gamma}$, $p_{j}+\delta e_{d}+\delta\mathbf{1}_{d\times1}\in \mathcal{S}_{\gamma}^c$.
By definition of $B_{\delta}$, this implies $p_{j}+\delta e_{d}\in B_{\delta}$. This contradicts iii) in
the definition of $p_{j}$. By contradiction, we have shown that each
point in $B_{2\delta}\cap Z_{j}$ has $d-$th coordinate smaller than
the $d-$th coordinate of $p_{j}$. So the claim that $B_{2\delta}\cap Z_{j}\subset\mathcal{R}(p_{j})$
for any $j\in J_0$ is proved.

Then we consider those $j$ such that $j\in J-J_{0}$.
For any point $x\in Z_{j}$, the first $d-1$ coordinates
of $x+\delta\mathbf{1}_{d\times1}$ are at least $j_{1}\delta,j_{2}\delta,\cdots,j_{d-1}\delta$
respectively. Since $j\notin J_{0}$,  we have that $x+\delta\mathbf{1}_{d\times1}\notin B_{2\delta}$.
This implies $x+3\delta\mathbf{1}_{d\times1}\notin \mathcal{S}_{\gamma}^c$, or $x\notin B_{3\delta}$.
So we have shown that for any $j\notin J_{0}$, $B_{3\delta}\cap Z_{j}=\emptyset$.
This implies $B_{3\delta}$ has a partition given by $B_{3\delta}=\cup_{j\in J}\left(B_{3\delta}\cap Z_{j}\right)=\cup_{j\in J_{0}}\left(B_{3\delta}\cap Z_{j}\right)$.
Notice that $B_{3\delta}\subset B_{2\delta}$, from the result in
the preceding paragraph, we conclude that $B_{3\delta}\subset\cup_{j\in J_{0}}\mathcal{R}(p_{j})$.

For each $j\in J_{0}$ and the constructed $p_{j}$, consider the
region 
\[
G_{j}:=\{x\in S_{\gamma}^{c}:x\geq p_{j}\}.
\]
Observe that, if there exists a sample point in $T_0$ that lies in $G_{j}$,
then we have $p_{j}\subset\mathcal{H}(T_{0})$ which implies $\mathcal{R}(p_{j})\subset\mathcal{H}(T_{0})$.
Since $p_{j}\in B_{\delta}$ and $\mathcal{S}_{\gamma}$ is orthogonally monotone, we have that $G_j$ contains the rectangle which contains $p_j$ and $p_j+\delta \mathbf{1}_{d\times 1}$ as two of its corners, so $\text{Vol}(G_{j})\ge\delta^{d}$.
Hence the probability that $\mathcal{R}(p_{j})\subset\mathcal{H}(T_{0})$
has a lower bound given by 
\[
P(\mathcal{R}(p_{j})\subset\mathcal{H}(T_{0}))\geq P(T_0 \cap G_j \neq \emptyset)\geq1-\left(1-\delta^{d}q_{l}\right)^{n_1}\geq1-e^{-n_1q_{l}\delta^{d}}.
\]
Notice that $\left|J_{0}\right|\leq\left(\frac{M}{\delta}\right)^{d-1}$,
by union bound we have that 
\[
P(\cup_{j\in J_{0}}\mathcal{R}(p_{j})\subset\mathcal{H}(T_{0}))\geq1-\frac{M^{d-1}}{\delta^{d-1}}e^{-n_1q_{l}\delta^{d}}.
\]
Since we have shown that $B_{3\delta}\subset\cup_{j\in J_{0}}\mathcal{R}(p_{j})$,
this implies 
\[
P(B_{3\delta}\subset\mathcal{H}(T_{0}))\geq1-\frac{M^{d-1}}{\delta^{d-1}}e^{-n_1q_{l}\delta^{d}}.
\]
Based on this inequality, it is not hard to check that for $t(\delta,n_1)=3\left(\frac{\log(n_1q_{l})+d\log{M}+\log\frac{1}{\delta}}{n_1q_{l}}\right)^{\frac{1}{d}}$, we have that $P(B_{t(\delta)}\subset\mathcal{H}(T_0))\geq 1-\delta$. 
\end{proof}

\begin{proof}[Proof of Theorem \ref{thm: set_OMhull}]
First, we show the inequality in the theorem, i.e., $P_{X\sim q}(X\in\mathcal{H}(T_0)^c\backslash\mathcal{S}_{\gamma})\leq M^{d-1}q_u\left(\frac{\sqrt{d}}{2}\right)^{d-1}w_{d-1}t(\delta,n_1)$. It suffices to show that with probability at least $1-\delta$, $\text{Vol}\left(\mathcal{H}(T_{0})^{c}\backslash\mathcal{S}_{\gamma}\right)\leq M^{d-1}\left(\frac{\sqrt{d}}{2}\right)^{d-1}w_{d-1}t(\delta,n_1)$,
or equivalently $\text{Vol}(\mathcal{S}_{\gamma}^{c}\backslash\mathcal{H}(T_{0}))\leq M^{d-1}\left(\frac{\sqrt{d}}{2}\right)^{d-1}w_{d-1}t(\delta,n_1)$.
Since by lemma \ref{lem: distance} we have that $B_{t(\delta,n_1)}\subset\mathcal{H}(T_{0})$
with probability at least $1-\delta$, it suffices to show that $\text{Vol}(\mathcal{S}_{\gamma}^{c}\backslash B_{t(\delta,n_1)})\leq M^{d-1}\left(\frac{\sqrt{d}}{2}\right)^{d-1}w_{d-1}t(\delta,n_1)$.
This latter inequality actually follows from the definition of $B_{t(\delta,n_1)}$
and some geometric argument. Indeed, by definition of $B_{t(\delta,n_1)}$,
for each $x\in\mathcal{S}_{\gamma}^{c}\backslash B_{t(\delta,n_1)}$,
$x$ belongs to the area which is obtained by moving the boundary
of $\mathcal{S}_{\gamma}$ in direction $-\frac{\mathbf{1}_{d\times1}}{\sqrt{d}}$
for a distance of $t(\delta,n_1)\sqrt{d}$. So the volume of $\mathcal{S}_{\gamma}^{c}\backslash B_{t(\delta,n_1)}$
is bounded by 
\begin{align*}
 & t(\delta,n_1)\sqrt{d}\times\text{Vol}_{d-1}(\text{projection of the boundary of \ensuremath{S_{0}} in direction}\ \mathbf{1}_{d\times1})\\
\leq & t(\delta,n_1)\sqrt{d}\times\text{Vol}_{d-1}(\text{projection of \ensuremath{[0,M]^{d}} in direction}\ \mathbf{1}_{d\times1})
\end{align*}
Here $\text{Vol}_{d-1}$ means computing volume in the $d-1$ dimensional
space. Notice that $[0,M]^{d}$ is contained in a ball with radius
$\frac{M\sqrt{d}}{2}$, we have that 
\[
\text{Vol}_{d-1}(\text{projection of \ensuremath{[0,M]^{d}} in direction}\ \mathbf{1}_{d\times1})\leq M^{d-1}\left(\frac{\sqrt{d}}{2}\right)^{d-1}w_{d-1}.
\]
Combining the preceding two inequalities, we have proved the inequality in the theorem. Next we show the equality in the theorem. Indeed, when $d$ is large, we have asymptotic formula
$w_{d}=\frac{1}{\sqrt{d\pi}}\left(\frac{2\pi e}{d}\right)^{\frac{d}{2}}(1+O(d^{-1}))$.
Plugging this into the RHS above, we will obtain the asymptotic bound
as stated in the theorem. \end{proof}
 
\begin{proof}[Proof of Theorem \ref{thm: set_ERM}]
By Markov inequality and the definition of $h$,$\bar{\mathcal{S}}_{\gamma}^{\hat{\kappa}}$, we
know that 
\begin{equation}
P_{X\sim q}\left(X\in\bar{\mathcal{S}}_{\gamma}^{\hat{\kappa}},X\in\mathcal{S}_{\gamma}^{c}\right)= P_{X\sim q}(\hat{g}(X)\geq\hat\kappa,X\in{\mathcal{S}_{\gamma}^c})\leq\frac{R(\hat{g})}{h(\hat{\kappa})}.\label{eq: Markov}
\end{equation}
We will compare the numerator and denominator of the RHS of (\ref{eq: Markov}) with their counterparts for the true minimizer $g^*$. For the numerator, since $\hat{g}$ is the empirical risk minimizer, we have that 
\begin{align*}
R(\hat{g}) & \leq R_{n_1}(g)+\sup_{g_{\theta}\in\mathcal{G}}\left|R_{n_1}(g_{\theta})-R(g_{\theta})\right|\leq R_{n_1}(g^{*})+\sup_{g_{\theta}\in\mathcal{G}}\left|R_{n_1}(g_{\theta})-R(g_{\theta})\right|\\
 & \leq R(g^{*})+2\sup_{g_{\theta}\in\mathcal{G}}\left|R_{n_1}(g_{\theta})-R(g_{\theta})\right|.
\end{align*}
For the denominator, from the definition of $\bar{\mathcal{S}}_{\gamma}^{\hat{\kappa}}$, it is not hard to verify that, in Algorithm \ref{algo:stage1}, our choice of $\hat{\kappa}$ is given by $\hat{\kappa}=\min\{\hat{g}(x):x\in\mathcal{H}(T_0)^c\}$. By lemma \ref{lem: distance}, we have that with probability at least
$1-\delta$, $B_{t(\delta,n_1)}\subset\mathcal{H}(T_{0})$, which implies
that with probability at least $1-\delta$, 
\begin{align*}
\hat{\kappa} & \geq\min\{\hat{g}(x):x\in B_{t(\delta,n_1)}^{c}\}\geq\min\{g^{*}(x):x\in B_{t(\delta,n_1)}^{c}\}-\left\Vert \hat{g}-g^{*}\right\Vert _{\infty}\\
 & \geq\min\{g^{*}(x):x\in \mathcal{S}_{\gamma }\}-t(\delta,n_1)\sqrt{d}\text{Lip}(g^{*})-\left\Vert \hat{g}-g^{*}\right\Vert _{\infty}\\
 & =\kappa^{*}-t(\delta,n_1)\sqrt{d}\text{Lip}(g^{*})-\left\Vert \hat{g}-g^{*}\right\Vert _{\infty}.
\end{align*}
Putting the preceding two inequalities into the Markov inequality
(\ref{eq: Markov}), and notice that $h$ is non decreasing by its definition, the theorem is proved. \end{proof}

\end{document}